\PassOptionsToPackage{table,svgnames}{xcolor}
\documentclass[12pt]{article}
\usepackage{amsmath}
\usepackage{amsfonts}
\usepackage{amsthm}
\usepackage{amssymb}
\usepackage{bbm}
\usepackage{url}
\usepackage{graphicx}
\usepackage{float}
\usepackage{algorithm}
\usepackage{authblk}
\usepackage{algpseudocode}
\usepackage{mathtools}
\usepackage{bm}
\usepackage{dsfont}
\usepackage[table,svgnames]{xcolor}
\usepackage{xspace}
\usepackage{thmtools}
\usepackage{thm-restate}
\usepackage{multirow}
\usepackage{enumitem}

\newcommand{\Var}{\operatorname{Var}}
\newcommand{\Cov}{\operatorname{Cov}}

\definecolor{naturegray}{RGB}{245,245,245} 
\definecolor{darkgray}{RGB}{100,100,100}

\usepackage[most]{tcolorbox}

\newtcolorbox{methodbox}{
  breakable,
  enhanced,
  colback=naturegray,
  colframe=darkgray,
  boxrule=0.5pt,
  arc=2pt,
  left=10pt,
  right=10pt,
  top=10pt,
  bottom=10pt,
  before skip=10pt,
  after skip=10pt
}

\usepackage{color-edits} 
\addauthor[Vikram]{vk}{blue}
\addauthor[Argyris]{ao}{limeGreen}
\addauthor[Manolis]{mz}{teal}

\newcommand{\notshow}[1]{}

\newcommand{\E}{\mathbb{E}}

\newcommand{\partest}{\textmd{PART}\xspace}
\newcommand{\parq}{\textmd{PAQ}\xspace}
\newcommand{\ppi}{\textmd{PPI}\xspace}
\newcommand{\ppiplus}{\textmd{PPI\raisebox{0.3ex}{\scriptsize ++}}\xspace}

\usepackage{hyperref}
\hypersetup{
  colorlinks = true, 
  urlcolor = {blueGrotto},
  linkcolor = {royalBlue},
  citecolor = {navyBlue}
}
\usepackage[capitalise,noabbrev,nameinlink]{cleveref}
\usepackage[table,svgnames]{xcolor} 
\definecolor{niceRed}{RGB}{190,38,38}
\definecolor{niceYellow}{HTML}{f5b400}
\definecolor{blueGrotto}{HTML}{059DC0}
\definecolor{royalBlue}{HTML}{057DCD}
\definecolor{navyBlue}{HTML}{0B579C}
\definecolor{limeGreen}{HTML}{81B622}
\definecolor{nicePurple}{HTML}{9c27b0}
\definecolor{lightRoyalBlue}{HTML}{def2ff}  
\definecolor{gold}{HTML}{ffa300}
\usepackage{caption} 

\usepackage{array} 
\usepackage{booktabs} 
\usepackage[numbers,sort&compress]{natbib}

\usepackage{authblk}
\usepackage[column=0]{cellspace} 
\usepackage{makecell}
\newtheorem{theorem}{Theorem}
\newtheorem{lemma}{Lemma}

\newtheorem{corollary}[lemma]{Corollary}

\newtheorem{remark}{Remark}

\newtheorem{example}{Example}
\newtheorem{assumption}{Assumption}

\crefname{theorem}{theorem}{theorems}
\Crefname{theorem}{Theorem}{Theorems}

\crefname{claim}{claim}{claims}
\Crefname{claim}{Claim}{Claims}

\crefname{observation}{observation}{observations}
\Crefname{observation}{Observation}{Observations}

\crefname{lemma}{lemma}{lemmas}
\Crefname{lemma}{Lemma}{Lemmas}

\crefname{corollary}{corollary}{corollaries}
\Crefname{corollary}{Corollary}{Corollaries}

\crefname{remark}{remark}{remarks}
\Crefname{remark}{Remark}{Remarks}

\crefname{estimator}{estimator}{estimators}
\Crefname{estimator}{Estimator}{Estimators}

\crefname{proposition}{proposition}{propositions}
\Crefname{proposition}{Proposition}{Propositions}

\crefname{example}{example}{examples}
\Crefname{example}{Example}{Examples}

\crefname{assumption}{assumption}{assumptions}
\Crefname{assumption}{Assumption}{Assumptions}

\crefname{fact}{fact}{facts}
\Crefname{fact}{Fact}{Facts}

\crefname{definition}{definition}{definitions}
\Crefname{definition}{Definition}{Definitions}

\usepackage{lineno}

\usepackage[main,final]{neurips_2026}
\title{Prediction-Augmented Trees for \\ Reliable Statistical Inference}

\author[1]{Vikram Kher\thanks{\texttt{vikram[dot]kher[at]yale.edu}}\hspace{0.5em}}
\author[ ]{Argyris Oikonomou\thanks{\texttt{ar[dot]economou[at]outlook.com}}}
\author[1]{Manolis Zampetakis \thanks{\texttt{manolis[dot]zampetakis[at]yale.edu}}\hspace{0.5em}}

\affil[1]{Department of Computer Science, Yale University}
\date{\today}

\begin{document}

\maketitle

\begin{abstract}
    Machine learning (ML) is increasingly central to scientific discovery, but translating ML predictions into rigorous scientific claims still requires statistical inference on gold-standard labeled data. A setting that has emerged across the sciences provides the analyst with a small labeled sample, a much larger unlabeled sample, and a pre-trained predictor, with the goal of constructing a confidence interval for a quantity of interest, as was formalized in \citet{ppi}. In this ML-driven scientific discovery literature, labeled data is typically the binding constraint on precision while computation is comparatively unconstrained, motivating the design question of how additional computation can extract more inferential value from a fixed labeled set, a fixed unlabeled set, and a fixed predictor.

To address this question, we propose the \emph{Prediction-Augmented Residual Tree} (\partest) estimator. \partest takes the same inputs as the original \ppi estimator and replaces its single global rectifier with locally computed rectifiers obtained from an adaptive partition of the feature space. We show that \partest outperforms existing methods by producing tighter confidence intervals across real-world datasets from ecology, astronomy, and census reports, among other domains. We then describe and analyze \parq, an estimator based on numerical quadrature techniques. Under appropriate assumptions on the input data, we show that the variance of \parq shrinks at a rate of $O(N^{-1} + n^{-4})$, showing that there are settings where the rate of $O(N^{-1}+n^{-1})$ of existing methods can be provably broken. 

\end{abstract}
\section{Introduction}

The integration of machine learning (ML) into scientific discovery has fundamentally transformed how researchers across disciplines generate knowledge. From predicting protein structures \citep{Jumper2021} to discovering new materials \citep{Merchant2023} to mapping poverty from satellite imagery \citep{Jean2016}, ML models now serve as the primary engines of scientific inference across the biological \citep{Stokes2020, Dauparas2022, Lin2023, Nguyen2024, Shen2024}, physical \citep{Shallue2018, Carleo2017, Lam2023}, and social sciences \citep{Kleinberg2018, Chernozhukov2018, Piech2015}. The 2024 Nobel Prize in Chemistry, awarded in part for AlphaFold \citep{Jumper2021, OLeary2024}, cemented what the scientific community already knew: ML is not a peripheral tool but a central pillar of modern scientific discovery.

Translating ML predictions into rigorous scientific claims, however, still requires statistical inference, and statistical inference still requires gold-standard labeled data. One fundamental setting that has emerged across these disciplines is the following: a researcher has access to a small labeled set $\mathcal{L} = \{(X_i, Y_i)\}_{i=1}^n$ in which $X_i$ is a feature vector that is inexpensive to obtain and $Y_i$ is a quantity of interest that is expensive to measure; a much larger unlabeled set $\mathcal{U} = \{\widetilde{X}_j\}_{j=1}^N$ with $N \gg n$; and a pre-trained black-box predictor $f$, such as AlphaFold. The objective is to combine these inputs into a confidence interval for a target functional of the joint distribution $P_{X,Y}$. This function can be the population mean, a regression coefficient, or an odds ratio. The goal is to obtain confidence intervals that are meaningfully tighter than the intervals obtained from the labeled data $\mathcal{L}$ alone. \citet{ppi} formalized this protocol as \emph{Prediction-Powered Inference} (\ppi), and \citet{ppi++} introduced the refinement \ppiplus.

The resources needed to perform such statistical inference have very different costs. A single labeled point $(X_i, Y_i)$ is very expensive and typically requires physical experimentation, dedicated data-collection effort, or expert annotation; an unlabeled covariate $\widetilde{X}_j$ is generally obtainable at a small fraction of this cost; and the predictor $f$ is taken as exogenous, with its training cost amortized across the broader research community that consumes its outputs; the computation time needed to compute the confidence interval is negligible compared to the rest. Throughout ML-driven scientific discovery, the labeled-sample size $n$ is therefore the binding constraint on inferential precision, while computation, measured on the same scale, is comparatively unconstrained. Tightening confidence intervals by enlarging the labeled set $\mathcal{L}$ requires new experimental effort and is rarely under the analyst's control; tightening them by spending additional computation on a fixed $\mathcal{L}$ is cheap by comparison. The natural design question is therefore the following, which is the main motivation of this paper.

\begin{center}
  \emph{\textbf{Question:} How can we use additional computation to extract more inferential value from the same labeled data, the same unlabeled data, and the same predictor?}
\end{center}

On its face, the existing \ppi/\ppiplus pipeline does not admit a natural place to invest additional computation. Both estimators take the form ``predictor plug-in plus residual rectifier,'' where the rectifier is a single global empirical mean of the form $\tfrac{1}{n}\sum_{i=1}^n (Y_i - \hat{\lambda} f(X_i))$, computed in $O(n)$ time. The global rectifier is, however, also the source of the largest unexploited statistical efficiency in these estimators. In essentially every scientific setting in which they are deployed, the predictor $f$ is non-uniformly accurate across the feature space: medical imaging models systematically underperform on rare subgroups while preserving high aggregate accuracy \citep{oakden-rayner2020}; covariate shift between training and deployment concentrates errors in feature-dependent regions \citep{WILDS2021}; structural biology predictors are known to vary in reliability across structural classes \citep{Jumper2021}. A single global rectifier averages over this heterogeneity and discards the associated signal. Recovering it requires a finer-grained rectifier, the construction of which is the natural site for additional computational investment.

\subsection{Our Contributions} \label{sec:contributions}

We propose the \emph{Prediction-Augmented Residual Tree} (\partest) estimator, which leverages additional computation to reduce the variance of \ppi/\ppiplus without altering their input requirements. \partest constructs a CART-style binary tree over the feature space whose splits are chosen greedily to minimize the variance contribution of the rectifier, and replaces \ppiplus's single global rectifier with a leaf-weighted mixture of locally computed rectifiers. The construction takes the same inputs $(\mathcal{L}, \mathcal{U}, f)$ as \ppi/\ppiplus and introduces no auxiliary trained model. For mean estimation, setting the tree depth to zero recovers \ppiplus exactly. Under our regularity conditions, the \partest mean estimator has asymptotic variance no larger than \ppiplus, with strict improvement when the limiting local rectifiers differ across leaves (\Cref{cor:part-variance-reduction}).

\noindent \textbf{Empirical Evaluation.} We evaluate \partest on six datasets that have become standard benchmarks in the prediction-powered-inference literature, drawn from ecology (deforestation rates from satellite imagery; \citealp{Bullock2020}), astronomy (spiral-galaxy fractions from Galaxy Zoo 2; \citealp{Willett2013}), socioeconomic statistics (median house values from California census data \citealp{Pace1997}, and an income-by-sex regression on ACS data \citealp{Ding2021}), food science (Vinho Verde wine quality; \citealp{wine_data}), and structural biology (phosphorylation odds ratios from AlphaFold predictions on UniProt; \citealp{UniProt_Consortium2015,Bludau2022}). Across all six benchmarks, \partest produces strictly tighter confidence intervals than \ppiplus while maintaining nominal coverage, with width reductions in the range of $4$--$20\%$ depending on the task and the labeled-sample size. These reductions are most informatively expressed in labeled-sample-equivalent terms: because the asymptotic variance scales as $n^{-1}$, a width reduction of $10$--$20\%$ corresponds to an effective increase in labeled-sample size of $23$--$56\%$. In application domains in which labeled samples are obtained through experimental procedures whose marginal cost is high, a gain of this magnitude obtained solely from additional computation is very valuable.

\noindent \textbf{Theoretical Results.} 
We establish asymptotic normality and asymptotically valid confidence intervals for the \partest mean and linear regression estimators (\Cref{thm:tree-ci-coverage} and \Cref{thm:asymptotic distribution for regression}) under regularity assumptions (\Cref{as:part-regularity,as:part-lr-regularity}, respectively). Our results allow the same labeled observations to be used both to learn the tree and to construct confidence intervals, without sample splitting. We account for this adaptive selection by showing that, under the stated regularity conditions, the fitted leaf regions converge to those of a deterministic population tree.

We complement \partest with a companion estimator, \emph{Prediction-Augmented Quadrature} (\parq), which exploits smooth, deterministic residuals to achieve faster convergence. The estimator approximates the debiasing term by interpolating labeled residuals and averaging the interpolant over the unlabeled sample, giving a numerical quadrature interpretation. Under \Cref{as:nn}, \Cref{thm:nn_uni} establishes a variance bound of $O(N^{-1}+n^{-4})$ for linear \parq. When $N \gg n$, this dramatically improves upon the usual $O(N^{-1}+n^{-1})$ variance bound of \ppi and \ppiplus. 

\section{Preliminaries}\label{sec:prelim}

In this section, we introduce the statistical framework and model that we use to present our theoretical results on $\partest$ and $\parq$ estimators. 

\noindent \textit{Empirical Expectation and Variance.}  
Let \(\{Z_i\}_{i=1}^n\) be a sequence of independently and identically distributed random variables (i.i.d.) with mean \(\mu = \E[Z]\) and variance \(\sigma^2 = \Var[Z] = \E[(Z - \mu)^2]\). The \emph{empirical mean} is defined as \(\widehat{\mu}_n = \frac{1}{n}\sum_{i=1}^n Z_i,\) and the \emph{empirical variance} is defined as \(\widehat{\sigma}^2_n = \frac{1}{n-1}\sum_{i=1}^n (Z_i - \widehat{\mu}_n)^2\). It is easy to see that $\E[\widehat{\mu}_n] = \mu$ and $\E[\widehat{\sigma}^2_n] = \sigma^2$, i.e., $\widehat{\mu}_n$ and $\widehat{\sigma}^2_n$ are \textit{unbiased} estimators of the mean and variance respectively. Lastly, for an integer $k\geq 0$, we let $C^k([a,b])$ denote the set of functions with continuous derivatives through order $k$ on the interval $[a,b]$. 

\noindent \textit{Convergence in Distribution.}  
A sequence of random variables \(\{Z_n\}_{n=1}^\infty\) is said to converge in distribution to a random variable \(Y\) (denoted by \(Z_n \xrightarrow{d} Y\)) if, for every real number \(x\) at which the \textit{cumulative distribution function} (CDF) \(F_Y(x)\) is continuous, the CDF \(F_{Z_n}(x)\) satisfies \(\lim_{n\to\infty} F_{Z_n}(x) = F_Y(x).\) For a cumulative distribution function $F_Y$, we define the quantile at level $u\in(0,1)$ by $F_Y^{-1}(u)\coloneqq \inf\{t\in\mathbb{R}:F_Y(t)\geq u\}.$ We write $Z_n=O_p(1)$ if $Z_n$ is bounded in probability, and $Z_n=o_p(1)$ if $Z_n$ converges to zero in probability.

\noindent \textit{Central Limit Theorem (CLT) and Confidence Intervals.}  
The Central Limit Theorem states that for i.i.d. random variables \(\{Z_i\}_{i=1}^n\) with finite variance \(\sigma^2\), the normalized empirical mean converges in distribution to a normal distribution, meaning that
$
    \sqrt{n}\left(\widehat{\mu}_n - \mu\right) \xrightarrow{d} \mathcal{N}(0, \sigma^2).
$

\subsection{Model} \label{sec:model}

Our datasets include two sets of data observations: \textit{labeled} and \textit{unlabeled} observations. 
\begin{itemize}[leftmargin=15pt]
  \item The \emph{labeled dataset} \(\mathcal{L}\) consists of \(n\) independent samples \(\{(x_i, y_i)\}_{i=1}^n\) drawn from the joint distribution \(P_{X,Y}\). Here, \(x_i \in \mathcal{X} \subseteq \mathbb{R}^d\) is a \(d\)-dimensional feature vector, and \(y_i \in \mathbb{R}\) is the corresponding outcome. 
  \item The \emph{unlabeled dataset} \(\mathcal{U}\) contains \(N\) feature vectors \(\{\widetilde{x}_j\}_{j=1}^N\) drawn from $P_X$ which is the marginal distribution on $X$ of \(P_{X,Y}\). We assume that \(N \gg n\), reflecting the practical scenario where unlabeled feature vectors are considerably easier to obtain than labels.
\end{itemize}

\noindent We further assume access to a machine learning predictor \(f:\mathcal{X} \rightarrow \mathbb{R}\) that produces predicted outcomes for given feature vectors. We treat $f$ as fixed and independent of the labeled and unlabeled sample. Finally, for any $x\in \mathbb{R}^d$, we denote the $k$-th coordinate  by $x^{(k)}$ for $k=1,\ldots,d$.

\paragraph{Mean Estimation and Linear Regression.}  
In this work, we address two fundamental estimation problems: mean estimation and linear regression. For mean estimation, our objective is to accurately recover the expected value of the response, \(\mu_Y = \E[Y]\), from the observed data. In the linear regression setting, we seek a linear model that minimizes the expected squared prediction error. Formally, we aim to find a parameter vector \(\theta^* \in \mathbb{R}^d\) such that
\[
    \theta^* \in \arg\min_{\theta \in \mathbb{R}^d}\E\left[\left(Y - \theta^\top X\right)^2\right],
\]
where \(X \in \mathbb{R}^d\) denotes the feature vector. These problems underpin our development of semi-supervised estimators that leverage both labeled and abundant unlabeled data to enhance estimation accuracy.

\subsection{PPI \& PPI\raisebox{0.3ex}{\scriptsize ++}: The Case of the Mean}\label{sec:ppiplus}
We review several competing approaches for prediction-powered inference in mean estimation, highlighting the key insights and underlying intuition behind each method.

\paragraph{Empirical Estimator.}
A common estimation problem is to estimate the mean outcome \(\mu = \E[Y]\) based on the labeled data. The empirical estimator is given by \(\widehat{\mu}_E = \frac{1}{n}\sum_{i=1}^n y_i,\) which satisfies
\[
    \sqrt{n}(\widehat{\mu}_E - \mu) \xrightarrow{d} \mathcal{N}\left(0, \Var(Y)\right).
\]
\paragraph{PPI Estimator.} Building on this approach, the \ppi estimator of \citet{ppi} integrates the predictive model \(f(x)\) with the unlabeled data. It is defined as
\[
    \widehat{\mu}_{\ppi} = \frac{1}{N}\sum_{j=1}^N f(\widetilde{x}_j) + \frac{1}{n}\sum_{i=1}^n \left[y_i - f(x_i)\right].
\]
The first term leverages predictions on the large unlabeled dataset, while the second term corrects for residual errors observed in the labeled data. Applying the CLT, the \ppi estimator satisfies
\[
    \sqrt{n}(\widehat{\mu}_{\ppi} - \mu) \xrightarrow{d} \mathcal{N}\left(0, \Var\left(Y - f(X)\right)\right).
\]

Assuming that \(N \gg n\) and that the predictor \(f(x)\) is sufficiently accurate (i.e., \(\Var\left(Y - f(X)\right) < \Var(Y)\)), the PPI estimator yields a narrower confidence interval for \(\mu\) than the empirical estimator while preserving the same coverage probability.

\paragraph{PPI\raisebox{0.3ex}{\scriptsize ++} Estimator.}
\citet{ppi++} introduced a tunable version of PPI defined by 
\[
    \widehat{\mu}_{\text{PPI\raisebox{0.2ex}{\scriptsize ++}}} = \frac{1}{N}\sum_{j=1}^N \widehat{\lambda} \cdot f(\widetilde{x}_j) + \frac{1}{n}\sum_{i=1}^n \left[y_i - \widehat{\lambda} \cdot f(x_i)\right],
\]
where $\widehat{\lambda} = \frac{\widehat{\Cov}_{n}(Y,f(X))}{(1+ n/N)\widehat{\Var}_{n+N}(f(X_i))}$. Like PPI, this estimator remains consistent and satisfies
\[
    \sqrt{n}(\widehat{\mu}_{\text{PPI\raisebox{0.2ex}{\scriptsize ++}}} - \mu) \xrightarrow{d} \mathcal{N}\left(0, \Var\left(Y - \lambda^* \cdot f(X)\right)\right),
\]
where $\lambda^* = \frac{\Cov(Y,f(X))}{\Var(f(X))}$. Unlike \ppi, it is straightforward to show that this method produces confidence intervals that are asymptotically at least as small as the empirical estimator. 

\section{Our Estimators}
We are now ready to describe the methodology and the coverage of the \partest and \parq estimators.

\subsection{The PART Estimator:\\
\hspace{1em}\parbox[t]{.9\textwidth}{\itshape\normalsize An Adaptive Partitioning Estimator}} \label{sec:part}

In this section, we describe the \partest estimator and discuss its application to building confidence intervals for the mean. In \Cref{sec:semi-supervised_tree_linear_regression}, we adapt \partest to estimate linear regression coefficients. At a high level, \partest uses a CART-style binary tree to partition the feature space, with splits chosen to greedily minimize an estimate of the variance of the debiasing term. Inside each leaf of the tree, we compute a local residual correction using labeled points in the leaf. We then produce a final debiasing term by taking a mixture of these local corrections based on the fraction of unlabeled data assigned to each leaf. In \Cref{sec:motex}, we provide a simple motivating example to illustrate how partitioning the residuals can reduce the variance of the debiasing term. Below, we give a short outline for how to build the \partest estimator and construct valid confidence intervals.

\begin{methodbox}
\textbf{PART Estimator: Construction and Inference}

\medskip

\noindent\textbf{1. Data, Coefficient, and Residuals.} 
Given labeled data $\mathcal{L} = \{(x_i,y_i)\}_{i=1}^n$, unlabeled data $\mathcal{U} = \{\widetilde{x}_j\}_{j=1}^N$ (with $N \gg n$), and a predictor $f:\mathcal{X}\to\mathbb{R}$, we first estimate the coefficient
\[
    \hat{\lambda} 
    \;\coloneqq\;
    \frac{\widehat{\mathrm{Cov}}_n\!\big(Y,f(X)\big)}
    {\bigl(1 + n/N\bigr)\,\widehat{\mathrm{Var}}_{n+N}\!\big(f(X)\big)},
\]
and define the \emph{$\hat{\lambda}$-residuals}
\[
    r_i^{(\hat{\lambda})} \;\coloneqq\; y_i - \hat{\lambda}\, f(x_i),
    \qquad i=1,\dots,n.
\]

\medskip

\noindent\textbf{2. Adaptive Partitioning.} 
We recursively partition the feature space $\mathcal X$ into disjoint axis-aligned regions (leaves) $\{\mathcal R_\ell\}_{\ell=1}^L$. Fix a maximum depth $D$, a minimum number of labeled observations per child $m\ge2$, and quantile levels $0<u_1<\cdots<u_K<1$, with $K\geq2$, independently of $n$ and $N$. For each coordinate $k$, compute the empirical quantiles $q_1^{(k)},\ldots,q_K^{(k)}$ at these levels from the full unlabeled sample, and define the candidate thresholds by
\[
    \mathcal S^{(k)}
    \coloneqq
    \{q_j^{(k)}:j=1,\ldots,K\}.
\]
We remove duplicate thresholds within each coordinate. These sets are computed once and used at every node. At a node with region $\mathcal{R}$, a candidate $(k,s)$, with $s\in\mathcal{S}^{(k)}$, produces the children
\[
    \mathcal{R}_L(k,s)\coloneqq\{x\in\mathcal{R}:x^{(k)}\leq s\},
    \qquad
    \mathcal{R}_R(k,s)\coloneqq\{x\in\mathcal{R}:x^{(k)}>s\}.
\]
Let $n_{(k,s,\mathrm{L})}$ and $n_{(k,s,\mathrm{R})}$ be the labeled counts in these children, and let $\widehat{\sigma}^2_{(k,s,\mathrm{L})}$ and $\widehat{\sigma}^2_{(k,s,\mathrm{R})}$ be their sample variances of the $\hat{\lambda}$-residuals. Let $p_{(k,s,\mathrm{L})}$ and $p_{(k,s,\mathrm{R})}$ be the corresponding proportions of the full unlabeled sample. Among candidates with at least $m$ labeled observations in each child, minimize the \emph{Variance-Minimizing Split} (VMS) criterion
\[
    \mathrm{VMS}(k,s) \coloneqq
    p_{(k,s,\mathrm{L})}^2 \frac{\widehat{\sigma}^2_{(k,s,\mathrm{L})}}{n_{(k,s,\mathrm{L})}}
    +
    p_{(k,s,\mathrm{R})}^2 \frac{\widehat{\sigma}^2_{(k,s,\mathrm{R})}}{n_{(k,s,\mathrm{R})}},
\]
using a fixed deterministic rule to break ties. Recurse on both children, stopping at depth $D$ or when no admissible candidate remains. Let $T$ denote the resulting data-dependent tree, with leaves $\{\mathcal{R}_\ell\}_{\ell=1}^{L}$ and $L\leq2^D$.

\medskip

\noindent\textbf{3. Estimation and Inference.} 
Let $n_\ell$ be the labeled count in leaf $\ell$, $\overline{r}^{(\hat{\lambda})}_\ell$ its mean $\hat{\lambda}$-residual and $p_\ell$ the unlabeled mass proportion in leaf $\ell$. Let $\widehat{\sigma}_{\ell,\hat{\lambda}}^2$ be the sample variance of the $\hat{\lambda}$-residuals in leaf $\ell$. The \partest estimator $\hat{\mu}_{T,\hat{\lambda}}$ and its variance estimate $\widehat{\sigma}_{T,\hat{\lambda}}^2$ are defined as:
\[
    \hat{\mu}_{T,\hat{\lambda}}
    \coloneqq 
    \underbrace{\frac{1}{N}\sum_{j=1}^N \hat{\lambda}\, f(\widetilde{x}_j)}_{\text{Base Prediction}} 
    \;+\; 
    \underbrace{\sum_{\ell=1}^L p_\ell \, \overline{r}^{(\hat{\lambda})}_\ell}_{\text{Bias Correction}}, \qquad
    \widehat{\sigma}_{T,\hat{\lambda}}^2 
    \coloneqq 
    \sum_{\ell=1}^L p_\ell^2 \,\frac{\widehat{\sigma}_{\ell,\hat{\lambda}}^2}{n_\ell}.
\]
The $(1-\alpha)$ Wald confidence interval is given by $\hat{\mu}_{T, \hat{\lambda}} \pm z_{1-\alpha/2}\,\widehat{\sigma}_{T,\hat{\lambda}}$, where $z_{1-\alpha/2}$ is the $1-\alpha/2$ standard normal quantile. We prove coverage guarantees for this interval in \Cref{thm:tree-ci-coverage} and formalize the procedure in \Cref{alg:semi-supervised_tree}.

\medskip

\noindent\textbf{Remark.} 
When the tree depth is zero, there is a single leaf $\mathcal{R}_1 = \mathcal{X}$, so $p_1 = 1$ and $\overline{r}^{(\hat{\lambda})}_1 = \frac{1}{n}\sum_{i=1}^n (y_i - \hat{\lambda} f(x_i))$. In this case, $\hat{\mu}_{T,\hat{\lambda}}$ reduces exactly to the \ppiplus estimator.
\end{methodbox}
We now establish the coverage result. Let $p(\mathcal{R}) = P_X(\mathcal{R})$ be the population probability of region $\mathcal{R}$. The population \ppiplus coefficient and its residual are 
\[
    \lambda^*\coloneqq\frac{\Cov(Y,f(X))}{\Var(f(X))},
    \qquad
    r^{(\lambda^*)}(x,y)\coloneqq y-\lambda^*f(x).
\]
Let $T_0$ be the population greedy tree obtained by using the corresponding population quantiles as candidate thresholds, admitting splits with two positive-probability children, and minimizing the population VMS criterion
\[
    \sum_{a\in\{L,R\}}
    p\!\left(\mathcal R_a(k,s)\right)
    \Var\!\left(
    r^{(\lambda^*)}(X,Y)\mid X\in\mathcal R_a(k,s)
    \right)
\]
at each node. The population procedure uses the same quantile levels and depth limit as its empirical counterpart and stops when no admissible split remains. Under the regularity conditions below, the fitted leaf regions converge to the leaves of $T_0$.

\begin{assumption}[Regularity for PART]
\label{as:part-regularity}
The dimension $d$, maximum depth $D$, minimum child size $m\geq2$, and finite grid of candidate quantile levels are fixed. Assume:
\begin{enumerate}[label=(\roman*),leftmargin=*]

\item The labeled observations are i.i.d. from $P_{X,Y}$, the unlabeled observations are i.i.d. from $P_X$, and the two samples are independent. As $n,N\to\infty$, $n/N\to0$. Furthermore, $\E[Y^2+f(X)^2]<\infty$ and $\Var(f(X))>0$.

\item Whenever the population procedure evaluates candidate splits at a node, let $\mathcal R_{0,L}$ and $\mathcal R_{0,R}$ denote the children of a candidate split. Define their empirical counterparts $\mathcal R_L$ and $\mathcal R_R$ by replacing every population quantile threshold defining these regions with its unlabeled empirical quantile at the same level and coordinate. For every such correspondence, let $n_a$ be the labeled count in
$\mathcal R_a$, and assume
\[
    P_X(\mathcal R_a\triangle\mathcal R_{0,a}) \xrightarrow{p}0, \qquad a\in\{L,R\},
\]
\[
    \Pr\!\left[\min_{a\in\{L,R\}}n_a\geq m\right] \longrightarrow
    \mathbf 1\!\left\{\min_{a\in\{L,R\}}p(\mathcal R_{0,a})>0\right\},
\]
where $\triangle$ denotes the symmetric difference between two sets.

\item At every splitting node of $T_0$, the population splitting criterion has a unique minimizer among the feasible candidates. Moreover,
\[
    v_0 \coloneqq \sum_{\mathcal R\text{ a leaf of }T_0} p(\mathcal R)\, \Var\!\left(r^{(\lambda^*)}(X,Y)\mid X\in\mathcal R \right) >0.
\]
Here $v_0$ is the probability-weighted within-leaf residual variance of the population tree.

\end{enumerate}
\end{assumption}

The sampling and moment conditions in \emph{(i)} are standard for asymptotic mean inference. The regime $n/N\to0$ makes unlabeled sampling uncertainty negligible at the $\sqrt{n}$ scale. The first part of \emph{(ii)} requires empirical and population candidate regions to disagree on the membership of a new feature vector with probability tending to zero. The second requires the minimum-child-size rule to agree with population feasibility. For the candidate thresholds defined above, condition \emph{(ii)} holds if the feature distribution has a density positive almost everywhere on a rectangular support. Alternatively, it holds for continuous marginal distributions under the stronger regime $n/\sqrt N\to0$. These sufficient conditions are justified in \Cref{sec:semi-supervised_tree}. Together with consistency of the split scores, the unique-minimizer requirement in \emph{(iii)} makes the greedy split choices stabilize. Finally, $v_0>0$ requires residual variation within at least one population leaf, ensuring a nondegenerate limiting normal distribution.
\begin{restatable}[Asymptotic coverage of PART]{theorem}{treecicoverage}
\label{thm:tree-ci-coverage}
Under \Cref{as:part-regularity}, the data-dependent tree $T$, coefficient $\hat{\lambda}$, and estimates constructed above satisfy
\[
    \sqrt{n}\bigl(\hat{\mu}_{T,\hat{\lambda}}-\mu\bigr) \xrightarrow{d}\mathcal{N}(0,v_0), \qquad n\widehat{\sigma}_{T,\hat{\lambda}}^2\xrightarrow{p}v_0,
\]
where $\mu=\E[Y]$. Consequently, for every fixed $\alpha\in(0,1)$,
\[
    \Pr\!\left[ \mu\in \left[\hat{\mu}_{T,\hat{\lambda}} -z_{1-\alpha/2}\widehat{\sigma}_{T,\hat{\lambda}}, \hat{\mu}_{T,\hat{\lambda}} +z_{1-\alpha/2}\widehat{\sigma}_{T,\hat{\lambda}} \right]\right]
    \longrightarrow1-\alpha.
\]
\end{restatable}
It is straightforward to characterize when the \partest interval is asymptotically smaller than the interval implied by \ppiplus.
\begin{corollary}[Asymptotic variance reduction]
\label{cor:part-variance-reduction}
Under \Cref{as:part-regularity}, the \partest mean estimator has
asymptotic variance no larger than \ppiplus. In particular,
\[
v_0\leq
\Var\!\left(r^{(\lambda^*)}(X,Y)\right).
\]
The inequality is strict if and only if at least two leaves of $T_0$
have different mean $\lambda^*$-residuals.
\end{corollary}

\begin{proof}
The depth-zero case of \Cref{thm:tree-ci-coverage} gives the
\ppiplus asymptotic variance
$\Var(r^{(\lambda^*)}(X,Y))$.
By the law of total variance,
\[
    \begin{aligned}
    &\Var\!\left(r^{(\lambda^*)}(X,Y)\right)-v_0\\
    &\quad=
    \sum_{\mathcal R\text{ a leaf of }T_0}
    p(\mathcal R)
    \left(
    \E\!\left[r^{(\lambda^*)}(X,Y)\mid X\in\mathcal R\right]
    -\E\!\left[r^{(\lambda^*)}(X,Y)\right]
    \right)^2
    \geq0.
    \end{aligned}
\]
All population leaves have positive probability, so the difference is positive exactly when their mean residuals are not all equal.
\end{proof}

\subsection{The PAQ Estimator:\\
\hspace{1em}\parbox[t]{.9\textwidth}{\itshape\normalsize Residual Interpolation and Numerical Quadrature}}
\label{sec:paq_method}
We now introduce PAQ, an estimator that can achieve faster convergence when the response variable is deterministic and the residual function is smooth. The estimator approximates the debiasing term by interpolating the labeled residuals and averaging the interpolant over the unlabeled sample. This approach is motivated by classical work on weighted Monte Carlo integration, which shows that exploiting smoothness can yield faster convergence than ordinary Monte Carlo averaging (\citet{Yakowitz1978}).

Below, we consider linear interpolation when the feature distribution is uniform on $[0,1]$. In \Cref{sec:NN-1}, we extend the estimator to higher-degree polynomial interpolation, obtaining faster interpolation error rates under stronger smoothness assumptions. We also establish the corresponding results for nonuniform feature distributions satisfying additional regularity conditions. 

\begin{assumption}[Regularity for Linear PAQ]\label{as:nn}
Assume:
\begin{enumerate}[label=(\roman*),leftmargin=*]
\item The labeled observations are i.i.d. from $P_{X,Y}$, the unlabeled observations are i.i.d. from $P_X=\operatorname{Unif}[0,1]$, and the two
samples are independent. Furthermore, $0<\Var(Y)<\infty$.
\item There is a deterministic function $y$ such that $Y=y(X)$ almost
surely and the residual function $r(x)=y(x)-f(x)$ belongs to $C^2([0,1])$.
\end{enumerate}
\end{assumption}

The sampling assumptions are standard for mean estimation. The assumption that the residual is a deterministic function of $x$ is restrictive; however, nonzero conditional label noise generally imposes a variance lower bound of order $n^{-1}$ for mean estimators, precluding the faster variance rates we seek.

\begin{methodbox}
\textbf{PAQ Estimator: Linear Residual Interpolation}

\medskip
\noindent\textbf{1. Residual Interpolation.}
Given $n\ge3$ labeled observations and $N\ge2$ unlabeled observations, sort the labeled features as $x_{(1)}<\cdots<x_{(n)}$ and let $r_i=y_{(i)}-f(x_{(i)})$. For $i=1,\ldots,n-1$, define
\[
    s_i\coloneqq\frac{r_{i+1}-r_i}{x_{(i+1)}-x_{(i)}},
    \qquad
    q_1(x)\coloneqq r_i+s_i(x-x_{(i)}),
    \quad x\in[x_{(i)},x_{(i+1)}].
\]
For $x<x_{(1)}$, extend the first line; for $x>x_{(n)}$, extend the last line. Thus, $q_1$ linearly interpolates the residuals between labeled features and extrapolates linearly at the two boundaries.

\medskip
\noindent\textbf{2. Estimation.}
Define the adjusted predictions and their average by
\[
    Z_j\coloneqq f(\widetilde x_j)+q_1(\widetilde x_j), \qquad \widehat\mu_{\mathrm{PAQ}}\coloneqq\frac1N\sum_{j=1}^N Z_j.
\]

\medskip
\noindent\textbf{3. Confidence Interval.}
Estimate the variance due to unlabeled averaging by
\[
    \widehat\sigma_{\mathrm{PAQ}}^2 \coloneqq\frac1{N(N-1)} \sum_{j=1}^N(Z_j-\widehat\mu_{\mathrm{PAQ}})^2.
\]
To account for interpolation error, fix $\kappa>1$, for example $\kappa=2$, and compute
\[
    \widehat L_2 \coloneqq 2\kappa\max_{1\leq i\leq n-2} \left| \frac{s_{i+1}-s_i}{x_{(i+2)}-x_{(i)}} \right|
\]
as an inflated estimate of $\|r''\|_\infty$. Let $t_0(x),t_1(x)$ be the two labeled feature values whose line is used to evaluate $q_1(x)$, and define
\[
    w_1(x) \coloneqq \frac{1}{2}\left|(x-t_0(x))(x-t_1(x))\right|, \qquad \widehat B_1 \coloneqq \frac{\widehat L_2}{N} \sum_{j=1}^N w_1(\widetilde x_j).
\]
Here $w_1(x)$ captures the effect of the labeled-feature spacing on interpolation error, and $\widehat B_1$ is the average interpolation-error allowance over the unlabeled sample. For a fixed $\alpha\in(0,1)$, use the interval
\[
    C_{n,N}\coloneqq
    \left[ \widehat\mu_{\mathrm{PAQ}} -z_{1-\alpha/2}\widehat\sigma_{\mathrm{PAQ}}-\widehat B_1,
    \quad
    \widehat\mu_{\mathrm{PAQ}} +z_{1-\alpha/2}\widehat\sigma_{\mathrm{PAQ}}+\widehat B_1 \right],
\]
where $z_{1-\alpha/2}$ is the $(1-\alpha/2)$ standard normal quantile. Under \Cref{as:nn}, \Cref{thm:nn_uni} establishes asymptotic coverage of at least $1-\alpha$ for this interval. Pseudocode is provided in \Cref{alg:parq}.
\end{methodbox}

\begin{restatable}[Linear PAQ under uniform sampling]{theorem}{parqmain}
\label{thm:nn_uni}
Under \Cref{as:nn}, as $n,N\to\infty$, with fixed $\kappa>1$ and
$\alpha\in(0,1)$, the estimator satisfies
\[
    \left|\E[\widehat\mu_{\mathrm{PAQ}}]-\mu\right|=O(n^{-2}),
    \qquad
    \E\!\left[(\widehat\mu_{\mathrm{PAQ}}-\mu)^2\right]
    =O(N^{-1}+n^{-4}),
\]
where $\mu=\E[Y]$. The same bound holds for its variance. Moreover,
\[
    \liminf_{n,N\to\infty}\Pr(\mu\in C_{n,N})\geq1-\alpha,
    \qquad
    |C_{n,N}|=O_p(N^{-1/2}+n^{-2}),
\]
where $|C_{n,N}|$ denotes the interval length.
\end{restatable}

When $n/N\to0$, the interval length is $o_p(n^{-1/2})$, thus improving on the convergence rate of \partest and \ppiplus. In \Cref{sec:NN-1}, we prove the result and extend it to nonuniform feature distributions and higher-degree interpolation.

\section{Experimental Results with Real-World Data}\label{sec:experiments}
In this section, we evaluate the empirical performance of \partest across a diverse set of datasets, ranging from ecological monitoring to census data analysis. We apply \partest to the task of mean estimation in \cref{sec:deforest,sec:galaxy,sec:wine,sec:housing}, to odds ratio estimation in \cref{sec:alphafold}, and to the estimation of linear regression coefficients in \cref{sec:lrcensus}. These experiments reveal a consistent pattern that \partest outperforms existing estimators.  We utilize \ppiplus as our primary baseline and omit the classical \ppi estimator, as \ppiplus consistently outperforms \ppi. We also compare the experimental performance of \parq on a synthetic dataset with smooth and deterministic residuals, showing that it also outperforms \ppiplus. We defer these results to \Cref{sec:NNExp}. For each dataset, we vary the number of labeled samples $n$ and evaluate both confidence interval width and the empirical coverage probability over repeated sub-sampling. We include error bars representing one standard deviation on each measured $n$ value in the interval width graph. We include detailed testing methodology and dataset descriptions in \Cref{sec:dataset_descriptions}. Tabulated data from the experiments can be found in \Cref{tab:master_results}.

\subsection{Predicting Mean Deforestation Rate from Satellite Imagery}\label{sec:deforest}
We first consider satellite imagery data capturing forest cover in the Amazon rainforest. We seek to measure the mean deforestation rate in the forest cover data provided by \citep{Bullock2020} using the \partest estimator, where the predictions are provided by a histogram-based gradient boosting prediction model. \Cref{fig:forest} demonstrates substantially tighter confidence intervals compared to \ppiplus. For example, given access to $n=800$ labeled samples, PART produces intervals that are 9.0\% smaller than \ppiplus (see \Cref{tab:master_results}). We further note that this example demonstrates that trees with non-zero depth can result in smaller confidence intervals while still remaining valid. 

\begin{figure}[H]
        \centering
        \begin{minipage}{0.5\textwidth}
        \centering
        \includegraphics[width=\linewidth]{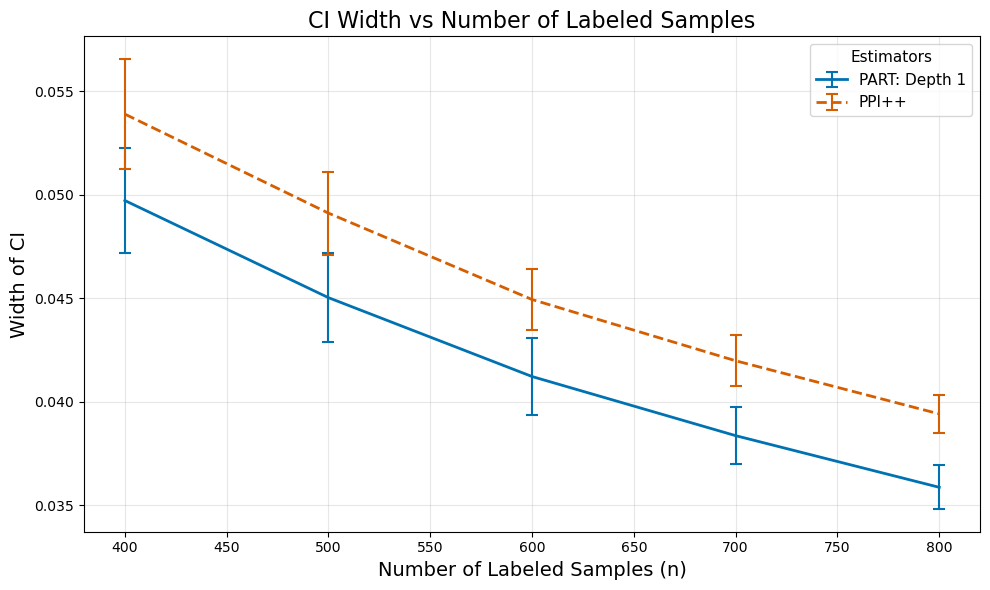}
        \end{minipage}\hfill
        \begin{minipage}{0.5\textwidth}
            \centering
            \includegraphics[width=\linewidth]{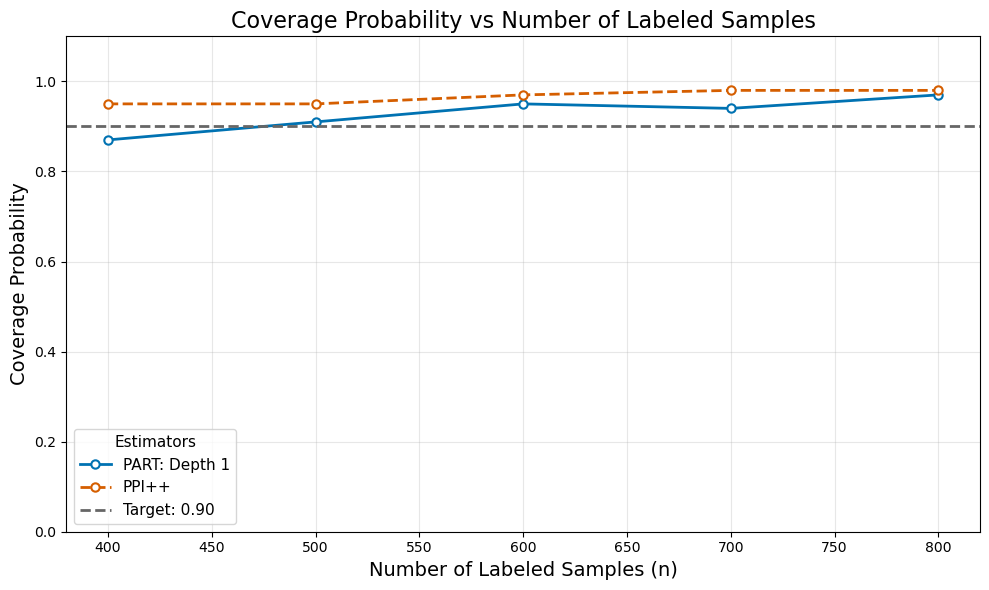}
        \end{minipage}
        \caption{Estimating the Amazon’s mean deforestation rate from satellite imagery \cite{Bullock2020}.}
        \label{fig:forest}

         \vspace{1em}
\end{figure}

\subsection{Estimating Mean Fraction of Spiral Galaxies}\label{sec:galaxy}
We next evaluate \partest on the Galaxy Zoo 2 dataset, which provides morphological classifications for galaxies in the Sloan Digital Sky Survey. Using auxiliary predictions from a ResNet50 model supplied by \citep{ppi}, \partest produces substantially tighter confidence intervals for the mean fraction of spiral galaxies compared to \ppiplus (\Cref{fig:galaxy}). For $n=500$, PART yields 4.5\% smaller intervals compared to \ppiplus (see \Cref{tab:master_results}). As in the previous experiment, we show that increasing the depth of our decision tree beyond zero improves estimation performance.
\begin{figure}[H]
        \centering
        \begin{minipage}{0.5\textwidth}
        \centering
        \includegraphics[width=\linewidth]{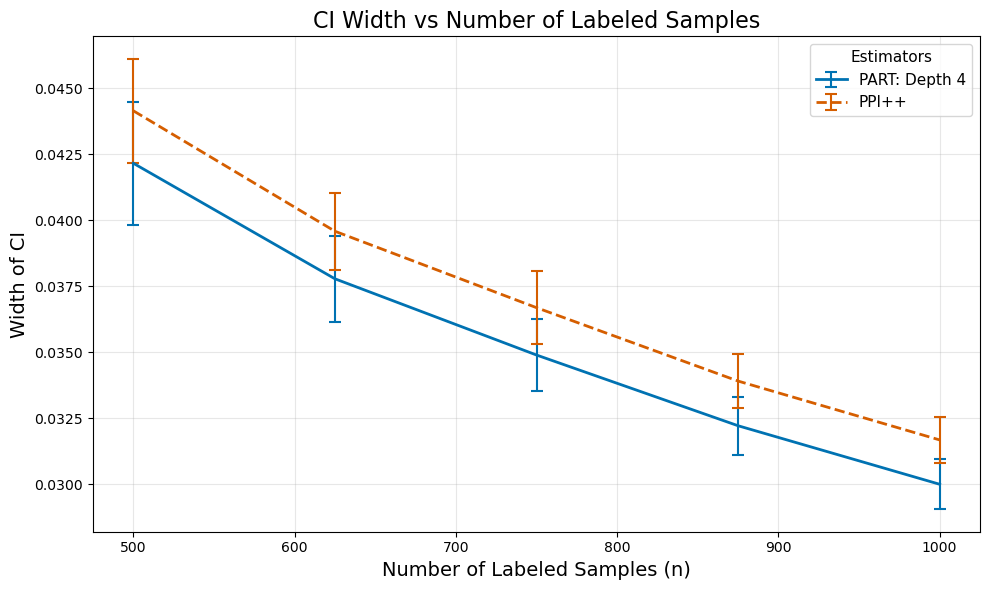}
        \end{minipage}\hfill
        \begin{minipage}{0.5\textwidth}
            \centering
            \includegraphics[width=\linewidth]{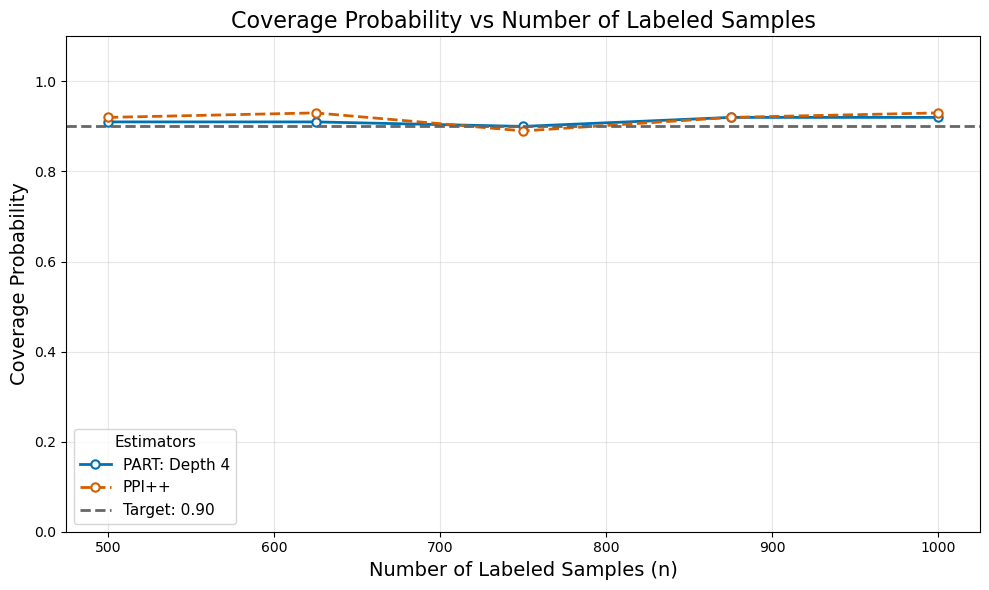}
        \end{minipage}
        \caption{Estimating the mean fraction of spiral galaxies using Galaxy Zoo 2 data \cite{Willett2013}.}
        \label{fig:galaxy}
\end{figure}

\subsection{Measuring Mean Quality of Portuguese Wine}\label{sec:wine}
We next turn to the task of estimating the mean quality of Portuguese ``Vinho Verde'' wine based on physicochemical tests \citep{wine_data}. Using auxiliary predictions from a Random Forest Classifier, \partest produces tighter confidence intervals for the mean quality score compared to \ppiplus (\Cref{fig:wine}). In particular, for $n=300$, we report 17.3\% smaller intervals compared to \ppiplus (see \Cref{tab:master_results}). 
\begin{figure}[H]
        \centering
        \begin{minipage}{0.5\textwidth}
        \centering
        \includegraphics[width=\linewidth]{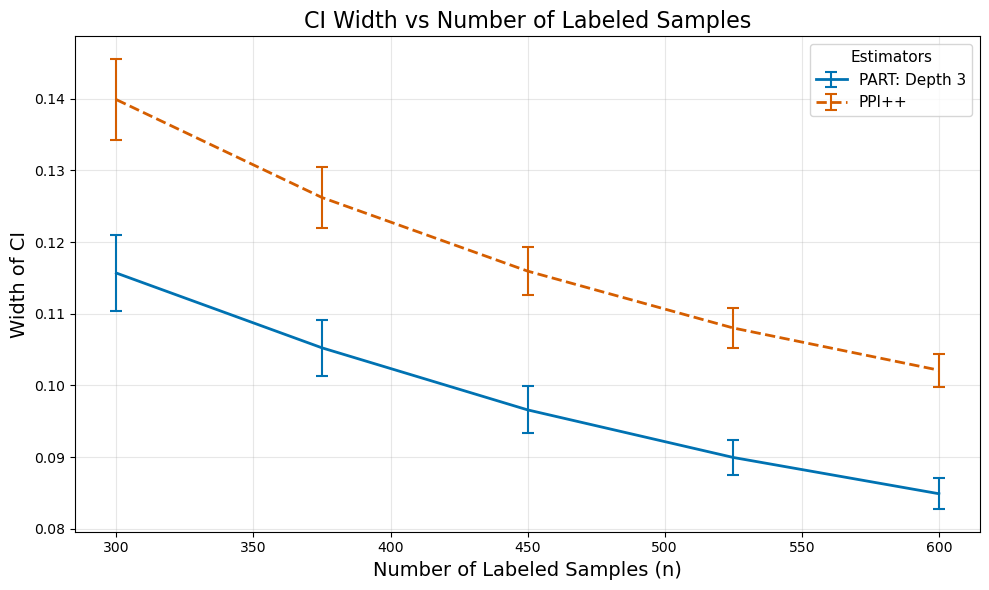}
        \end{minipage}\hfill
        \begin{minipage}{0.5\textwidth}
            \centering
            \includegraphics[width=\linewidth]{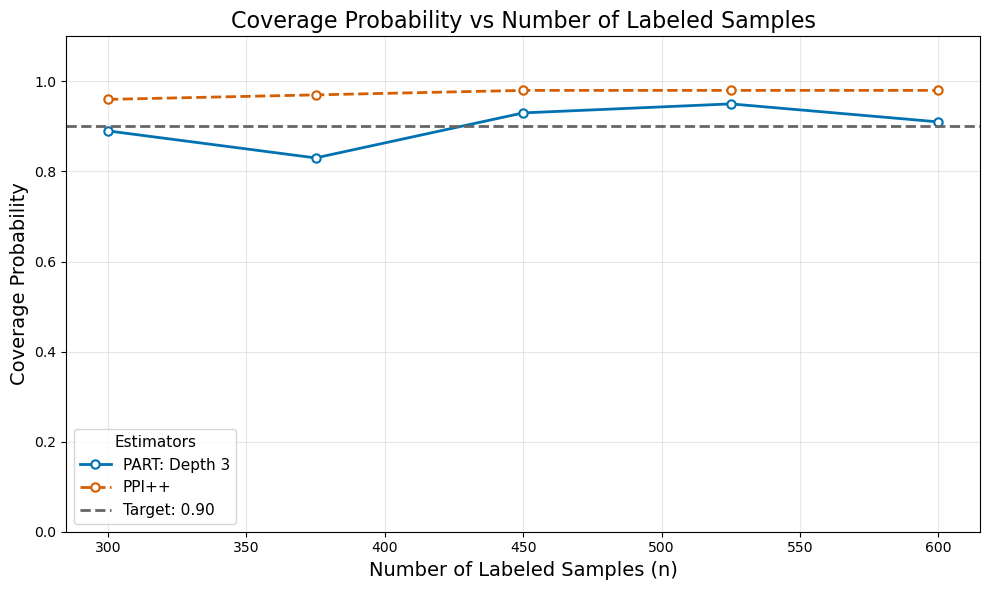}
        \end{minipage}
        \caption{Estimating the mean quality of Portuguese Vinho Verde wine \cite{wine_data}.}
        \label{fig:wine}
\end{figure}

\subsection{Estimating Average Median House Price Using Census Data}\label{sec:housing}
We next consider the California Housing dataset \citep{Pace1997}, a standard benchmark derived from the 1990 US Census. Our objective is to estimate the average median house value across census blocks. Leveraging predictions from a Random Forest Regressor, \partest achieves significantly tighter confidence intervals than \ppiplus (\Cref{fig:house}). As shown in \Cref{tab:master_results}, with $n=625$ labeled samples, \partest reduces the interval width by 13.6\% compared to the baseline.
\begin{figure}[H]
        \centering
        \begin{minipage}{0.5\textwidth}
        \centering
        \includegraphics[width=\linewidth]{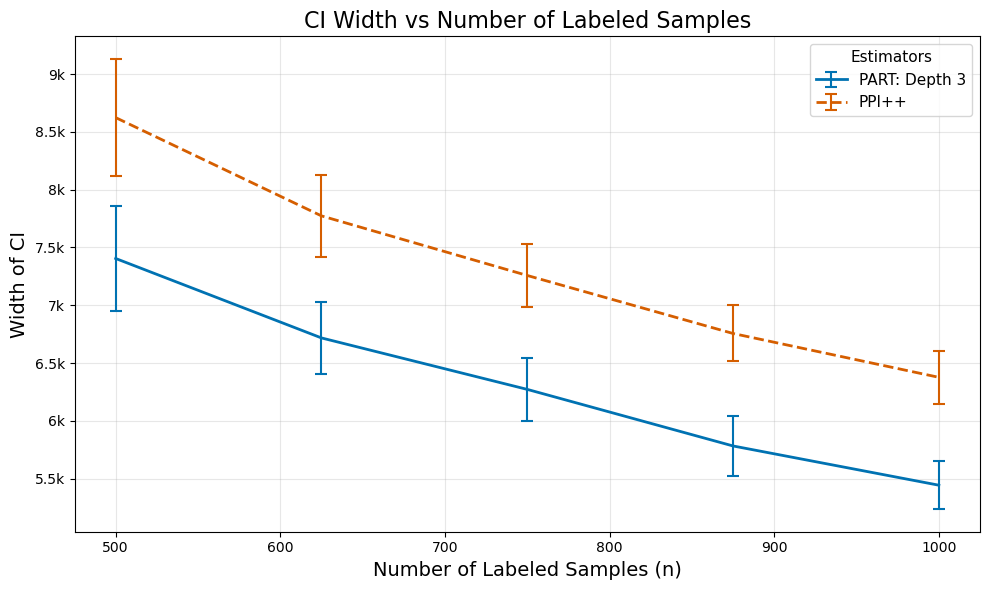}
        \end{minipage}\hfill
        \begin{minipage}{0.5\textwidth}
            \centering
            \includegraphics[width=\linewidth]{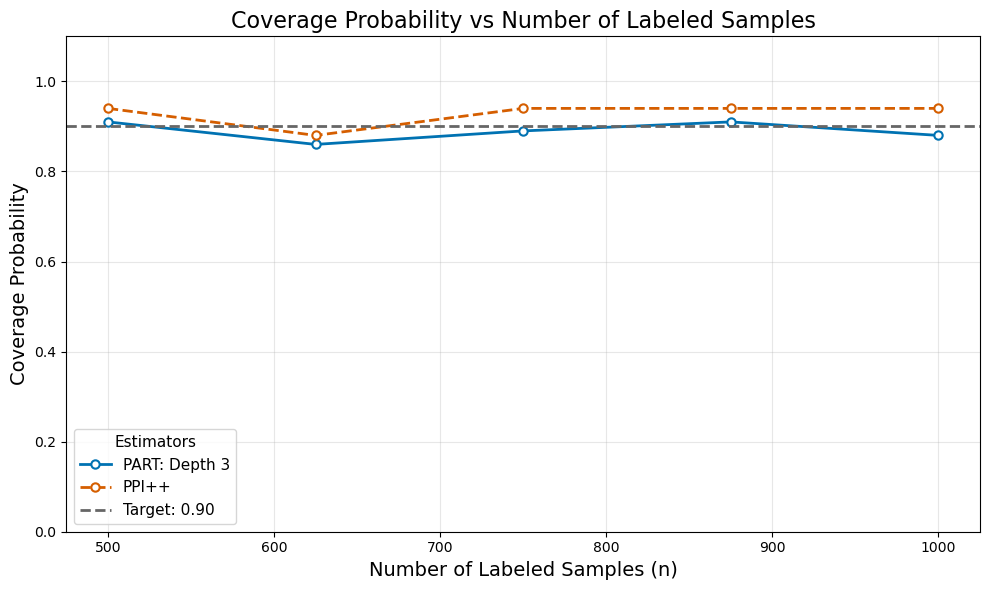}
        \end{minipage}
        \caption{Estimating the average median house price using California census data \cite{housing_data}.}
        \label{fig:house}
\end{figure}

\subsection{Estimating Odds Ratio of Protein Structures via AlphaFold}\label{sec:alphafold}
We next apply \partest to estimate the odds ratio of protein phosphorylation in intrinsically disordered regions (IDRs), a structural property recently analyzed by \citet{Bludau2022}. Using AlphaFold-derived structural predictions provided by \citep{ppi}, we estimate this odds ratio on data from the UniProt Consortium \citep{UniProt_Consortium2015}. As shown in \Cref{fig:alphafold}, \partest yields tighter confidence intervals compared to the baseline\footnote{For the precise derivation of the confidence interval for the odds ratio via the delta method, see \citep{ppi++}.}.

\begin{figure}[H]
        \centering
        \begin{minipage}{0.5\textwidth}
        \centering
        \includegraphics[width=\linewidth]{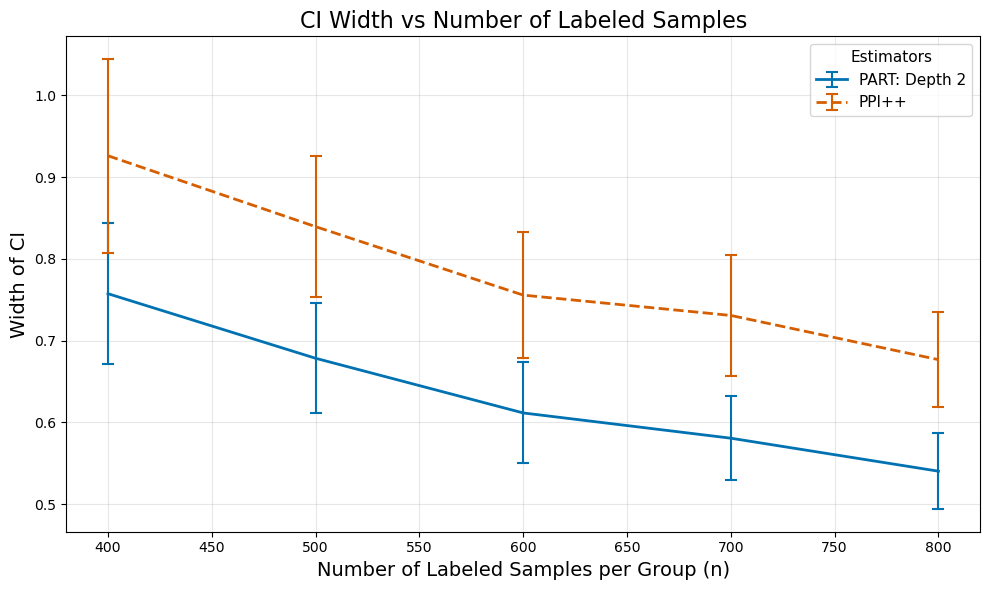}
        \end{minipage}\hfill
        \begin{minipage}{0.5\textwidth}
            \centering
            \includegraphics[width=\linewidth]{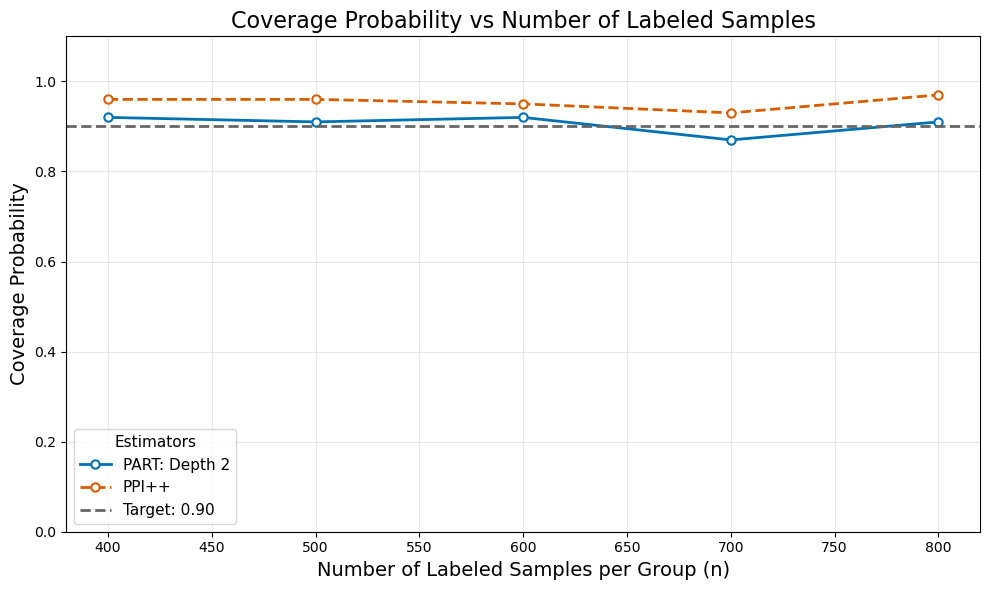}
        \end{minipage}
        \caption{Estimating the odds ratio of a protein being phosphorylated and belonging to an intrinsically disordered region using AlphaFold predictions \cite{UniProt_Consortium2015}.}
        \label{fig:alphafold}
\end{figure}

\bibliographystyle{abbrvnat}
\bibliography{references.bib}
\clearpage

\appendix
\section{Tabulated Experimental Results for PART}
Table \ref{tab:master_results} summarizes the performance of \partest and \ppiplus across all datasets. Housing price and census income values are reported in thousands of dollars. We report the mean confidence interval width $\pm$ standard deviation.

\begin{table}[H]
\centering
\small
\setlength{\tabcolsep}{5pt}
\renewcommand{\arraystretch}{1.02}
\caption{Comparison of Confidence Interval Widths: \partest vs. \ppiplus}
\label{tab:master_results}
\begin{tabular}{@{}lrccc@{}}
\toprule
Dataset & $n$ & PART width & PPI++ width & Reduction (\%) \\
\midrule
\multirow{5}{*}{Deforestation} & 400 & $0.050 \pm 0.003$ & $0.054 \pm 0.003$ & $7.7$ \\
 & 500 & $0.045 \pm 0.002$ & $0.049 \pm 0.002$ & $8.3$ \\
 & 600 & $0.041 \pm 0.002$ & $0.045 \pm 0.001$ & $8.3$ \\
 & 700 & $0.038 \pm 0.001$ & $0.042 \pm 0.001$ & $8.6$ \\
 & 800 & $0.036 \pm 0.001$ & $0.039 \pm 0.001$ & $9.0$ \\
\midrule
\multirow{5}{*}{Galaxy Zoo} & 500 & $0.042 \pm 0.002$ & $0.044 \pm 0.002$ & $4.5$ \\
 & 625 & $0.038 \pm 0.002$ & $0.040 \pm 0.001$ & $4.6$ \\
 & 750 & $0.035 \pm 0.001$ & $0.037 \pm 0.001$ & $4.9$ \\
 & 875 & $0.032 \pm 0.001$ & $0.034 \pm 0.001$ & $5.0$ \\
 & 1000 & $0.030 \pm 0.001$ & $0.032 \pm 0.001$ & $5.3$ \\
\midrule
\multirow{5}{*}{Wine Quality} & 300 & $0.116 \pm 0.005$ & $0.140 \pm 0.006$ & $17.3$ \\
 & 375 & $0.105 \pm 0.004$ & $0.126 \pm 0.004$ & $16.6$ \\
 & 450 & $0.097 \pm 0.003$ & $0.116 \pm 0.003$ & $16.7$ \\
 & 525 & $0.090 \pm 0.003$ & $0.108 \pm 0.003$ & $16.7$ \\
 & 600 & $0.085 \pm 0.002$ & $0.102 \pm 0.002$ & $16.9$ \\
\midrule
\multirow{5}{*}{Housing Prices} & 500 & $7.405 \pm 0.454$ & $8.624 \pm 0.509$ & $14.1$ \\
 & 625 & $6.717 \pm 0.310$ & $7.774 \pm 0.352$ & $13.6$ \\
 & 750 & $6.272 \pm 0.274$ & $7.258 \pm 0.269$ & $13.6$ \\
 & 875 & $5.782 \pm 0.258$ & $6.756 \pm 0.242$ & $14.4$ \\
 & 1000 & $5.443 \pm 0.208$ & $6.376 \pm 0.228$ & $14.6$ \\
\midrule
\multirow{5}{*}{AlphaFold (Odds Ratio)} & 400 & $0.757 \pm 0.086$ & $0.926 \pm 0.119$ & $18.2$ \\
 & 500 & $0.678 \pm 0.068$ & $0.839 \pm 0.086$ & $19.2$ \\
 & 600 & $0.611 \pm 0.062$ & $0.756 \pm 0.077$ & $19.1$ \\
 & 700 & $0.581 \pm 0.052$ & $0.731 \pm 0.074$ & $20.6$ \\
 & 800 & $0.540 \pm 0.047$ & $0.677 \pm 0.058$ & $20.2$ \\
\midrule
\multirow{5}{*}{Census Income (Linear Regression)} & 300 & $9.408 \pm 2.696$ & $9.766 \pm 2.799$ & $3.7$ \\
 & 400 & $8.116 \pm 1.997$ & $8.508 \pm 2.144$ & $4.6$ \\
 & 500 & $7.078 \pm 1.622$ & $7.422 \pm 1.732$ & $4.6$ \\
 & 600 & $6.435 \pm 1.301$ & $6.756 \pm 1.403$ & $4.8$ \\
 & 700 & $5.939 \pm 1.007$ & $6.258 \pm 1.063$ & $5.1$ \\
\bottomrule
\end{tabular}
\end{table}

\section{Experimental Methodology}\label{sec:dataset_descriptions}
In this section, we provide an overview of the testing methodology, tree hyperparameters, and datasets used in the experimental results. 

\paragraph{Testing Methodology.} For our testing methodology, we assume that each dataset has $M$ i.i.d. samples consisting of features, labels, and predicted labels. For various fixed values of $n$, we randomly obfuscate $N=M-n$ of the labels to produce $N$ pieces of unlabeled data and $n$ pieces of labeled data. We subsequently run the \partest estimator using a manually tuned depth limit on the labeled and unlabeled data, recording the width of the confidence interval and whether the interval contained the true mean. We repeat the previous procedure 100 times for each fixed $n$ and compute the average confidence interval width and the coverage probability. We produce two graphs to show how the width of the intervals and coverage probability change as a function of $n$. For the interval width graph, we include one standard deviation error bars on each measured $n$ value. We show the nominal coverage probability with a constant dashed gray line.

\paragraph{Tree construction and hyperparameters.}
For the \partest experiments, we compute empirical quantiles at the fixed levels $0.05,0.10,\ldots,0.95$ for each partitioning feature using the inverse empirical CDF of the full unlabeled sample. We use these quantile values directly as candidate thresholds, remove duplicates within each coordinate, and reuse the same thresholds at every node. For the mean estimation and odds-ratio tasks, we require at least $m=5$ labeled observations in each child. For the linear regression experiment, we require $m=100$. We specify the \partest depth limits in each experimental plot.

\paragraph{Deforestation Dataset}
Deforestation in the Amazon rainforest disrupts local ecosystems and releases stored carbon \citep{Bullock2020,Lapola2023}. This motivates remote monitoring of forest cover through satellite imagery \citep{Sexton2013}. This dataset consists of satellite imagery derived from Landsat sensors. The objective is to measure the proportion of images that contain evidence of deforestation. There are 1,596 observations with binary labels that indicate the presence of forest cover loss. In our experiments, we use predictions generated from a histogram-based gradient-boosted tree trained by \citep{ppi}.
\paragraph{Galaxy Dataset}
The Galaxy Zoo 2 project \citep{Willett2013} is a crowd-sourced initiative focused on providing morphological classifications for 304,122 galaxies captured by the Sloan Digital Sky Survey \citep{York2000}. These human-annotated labels serve as key aids in helping scientists understand galaxy evolution (e.g., \citep{Cheung2013}) but require substantial manual effort. To mitigate this, machine learning approaches---such as the convolutional neural network of \citet{Dielman2015}---predict morphology directly from the images. One goal for this dataset is to estimate the mean fraction of spiral galaxies in the local universe. Our dataset consists of 16,743 total observations where each label is the estimated probability that the galaxy is spiral. We use predictions from a ResNet50 model supplied by \citep{ppi}.
\paragraph{Wine Quality Dataset}
Vinho Verde wine originates from the Minho Province of Portugal, where it is tightly regulated under the Denominação de Origem Controlada (DOC) system. To deter counterfeiting, regulators use physicochemical and sensory tests to assess wine quality. \citet{wine_data} provide a dataset of 1280 wine samples, each with 11 features (pH, sulphates, alcohol content, etc) gathered via physicochemical tests and a quality score on a 0-10 scale. This dataset has become a standard benchmark for ensemble tree methods (e.g., \citep{Ngo2022}). The objective is to estimate the average wine quality based on this physicochemical data. For predictions, we train a Random Forest Classifier on 20\% of the data and generate predicted scores for the remaining 80\% of the data.
\paragraph{Housing Price Dataset}
California's housing market is shaped by many factors, including its coastal geography and strict zoning laws. One source of information about the housing market comes from census data. \citet{Pace1997} initially used the 1990 California Census to form a dataset consisting of 16,512 instances with 9 characteristics (population, number of households, total rooms, etc.) and 1 outcome variable indicating the median house value in the block (a geographical unit used by the Census Bureau). Many textbooks (e.g., \citep{housing_data}) use this dataset as a standard ML benchmark. The objective for this dataset is to estimate the average median house price in a block group in California. For predictions, we train a Random Forest Regressor on 20\% of the data and generate predicted scores for the remaining 80\% of the data.
\paragraph{AlphaFold Dataset}
Since the groundbreaking result of \cite{Jumper2021}, AlphaFold has quickly become an essential tool used by researchers to allow the study of the structure of proteins on a mass scale (e.g., see \citep{Wilson2022, Tunyasuvunakool2021, Omidi2024} among many others). Predictions from AlphaFold were recently used by \citet{Bludau2022} to understand whether intrinsically disordered regions (IDRs) on proteins experience phosphorylation at a greater rate. This dataset from \citep{UniProt_Consortium2015} consists of 10,802 samples. The objective for this dataset is to predict the odds ratio for a protein being phosphorylated and belonging to an IDR. We utilize the AlphaFold-based predictions from \citep{ppi}.
\paragraph{Income Dataset}
A persistent correlation between sex and earnings is well documented in the economics literature \citep{Blau2017}. We utilize census data from \citet{Ding2021} for the task of estimating the linear relationship between sex (M/F) and individual income. The dataset consists of 380,091 observations with two features, sex (M/F) and age (0-99), and a label for the person's income. For predictions, we utilize an XGBoost model trained by \citep{ppi++}.

\paragraph{Synthetic Datasets for the PAQ Estimator}
In this section, we overview the synthetic dataset we use to test the \parq estimator against \ppiplus. While the \parq estimator assumes that the residuals are deterministic and smooth, there are many natural settings where the conditions of \Cref{as:smooth_p} hold exactly or in good approximation. For example, Anfinsen’s thermodynamic hypothesis \citep{Anfinsen1973} suggests that for many small globular proteins under appropriate conditions the amino-acid sequence uniquely determines its physical structure. Similarly, in physics, there are many equations in which the final state of the system is determined by its initial conditions, such as in the 2-D incompressible Navier-Stokes equation \citep{constantin1988}.

To construct our synthetic dataset, we let
\[
X \sim \mathcal{U}[0,\pi],\quad y(x) = x^2 + \sin(x),\quad f(x) = x^2,\quad r(x) = \sin(x).
\]
The residual function is deterministic and twice continuously differentiable, and thus meets \Cref{as:smooth_p}. We demonstrate the performance of \parq in \Cref{fig:paq}.

\section{Related Work}

\paragraph{ML Predictions for Mean Estimation.} We briefly survey the relevant landscape of using ML predictions for statistical inference. \citet{ppi} formalized the model for prediction powered statistical inference and proposed the \ppi estimator for mean estimation, linear regression coefficient estimation, odds ratio estimation, among other targets. In follow-up work, \citet{ppi++} introduced \ppiplus, which uses a power-tuning parameter to adaptively down-weight the influence of the predictions based on their accuracy, yielding an estimator that is asymptotically no less efficient than the empirical mean. \citet{Zhu2023} arrived at an equivalent estimator to \citep{ppi} from the self-training perspective within semi-supervised learning. More recently, \citet{ji2025} proposed RePPI, a recalibrated version of \ppi, that uses a flexible machine learning model to model $\E[Y|f(X),X]$, rather than fitting a linear relationship as in \ppiplus. As RePPI requires training an auxiliary model to learn this conditional expectation, we retain \ppi and \ppiplus as our primary benchmarks. The methodology of \partest is most similar to StratPPI \citep{Fisch2024stratppi}, an estimator that uses stratified sampling over \emph{predefined} strata to improve the efficiency of \ppi/\ppiplus. In contrast, \partest learns stratification adaptively from the labeled residuals. For a detailed survey of existing prediction-enhanced estimators, we refer the reader to \citet{song2026demystifyingppi}.

Several papers have since followed which improve the practical use of \ppi-style estimators. While earlier work assumes access to a pre-trained ML model, \citet{Zrnic2024} use cross-validation techniques to train the ML model from scratch utilizing the small labeled set of data. \citet{Zrnic2025} integrates the bootstrap method to prediction powered inference to facilitate applications where the central limit theorem justifications may be tenuous. While we assume access to a pre-trained model in this work, we remark that both cross-validation and bootstrap methods can be applied to our estimator using a formula similar to the previous papers.

\paragraph{Learning-Augmented Algorithms.}
Learning-augmented (LA) algorithms seek to enhance classical algorithms by incorporating the predictions of ML models trained on historical data. The framework has been successfully applied to many settings such as facility location \citep{Agrawal2022,Balkanski2025}, mechanism design \citep{Berger2024,Gkatzelis2025}, and scheduling problems \citep{Bampis2022,Dinitz2022}, among many others\footnote{See \citep{ALPSsite2025} for an up-to-date list of learning-augmented papers.}. To adapt the original framework of \citet{Lykouris2021} to our setting, LA algorithms for mean estimation should output a confidence interval that 1) is tighter than the empirical mean when the predictions are accurate (\textit{consistency}) and 2) retains the correct coverage probability when the predictions are arbitrarily poor (\textit{robustness}). Similarly to \citep{ppi,ppi++}, we characterize the asymptotic distribution of the \partest estimator, allowing us to build confidence intervals that satisfy both consistency and robustness.

\paragraph{Decision Tree for Mean Estimation.} Decision trees have a long history, beginning with their use in exploratory data analysis by \citet{Morgan1963}. \citet{Breiman2017} formalized modern regression trees for estimating conditional means with the CART algorithm. Subsequent work has extended and strengthened CART, e.g., unbiased-splitting methods \citep{Loh2002} and popular ensemble approaches including ``bagging'' \citep{Breiman1996} and random forests \citep{Breiman2001}. We contribute to this line of literature by adapting the CART node splitting criteria of \citep{Breiman2017} to build the \partest decision tree for mean residual estimation.

\paragraph{Monte Carlo Methods for Numerical Integration.} A complementary line of work studies replacing empirical averages with carefully weighted sums to accelerate the convergence of integral estimates. It is well-known that the integral of a function on the unit interval can be approximated via an empirical average at a convergence rate of $O(n^{-1/2})$. \citet{Yakowitz1978} were the first to show that, for twice-differentiable functions on the unit interval, the canonical $O(n^{-1})$ rate can be improved to $O(n^{-4})$ via quadrature methods. Many results (e.g., \citep{Haber1967,Philippe1997}) have since followed that observe convergence improvements through weighted sums. We adapt these quadrature methods to estimate the debiasing term in our semi-supervised setting where the data is not uniformly distributed on the unit interval.
\section{Algorithms}\label{appx:algorithms}

We present the \partest algorithms for mean estimation and linear
regression from \Cref{sec:semi-supervised_tree,sec:semi-supervised_tree_linear_regression}. Both algorithms select splits using labeled residuals and estimate leaf probabilities using the unlabeled sample. We also give the linear \parq algorithm from \Cref{sec:paq_method}, which interpolates the labeled residuals and averages the adjusted predictions over the unlabeled sample. In the following, we recall some notation used in the construction of \partest. 

\paragraph{Candidate Splits and Regional Statistics.}
For both \partest algorithms, the maximum depth $D$, minimum child size $m\ge2$, and quantile levels $0<u_1<\cdots<u_K<1$, with $K\ge2$, are fixed independently of $n$ and $N$. For each feature coordinate $j$, compute the empirical quantiles $q_1^{(j)},\ldots,q_K^{(j)}$ from the full unlabeled sample, as in \Cref{sec:semi-supervised_tree}, and set
\[
    \mathcal S^{(j)}
    \coloneqq
    \{q_t^{(j)}:t=1,\ldots,K\}.
\]
Remove duplicate thresholds within each coordinate. These sets are computed once and used at every node. For any region $\mathcal R$, write
\[
    n_{\mathcal R}
    \coloneqq
    \sum_{i=1}^n\mathbf 1\{x_i\in\mathcal R\},
    \qquad
    p_{\mathcal R}
    \coloneqq
    \widehat p_N(\mathcal R)
    =\frac{1}N\sum_{j=1}^N
    \mathbf 1\{\widetilde x_j\in\mathcal R\}.
\]
At a node $\mathcal R$, a candidate $(j,s)$ with $s\in\mathcal S^{(j)}$ produces
\[
    \mathcal R_L(j,s)
    =\{x\in\mathcal R:x^{(j)}\le s\},
    \qquad
    \mathcal R_R(j,s)
    =\{x\in\mathcal R:x^{(j)}>s\}.
\]
It is admissible if both children contain at least $m$ labeled observations. Ties between candidate splits are resolved using a fixed deterministic rule.

\begin{algorithm}[H]
\caption{\partest Mean Estimation and Inference}
\label{alg:semi-supervised_tree}
\small
\begin{algorithmic}[1]
\Require Labeled data $\mathcal L=\{(x_i,y_i)\}_{i=1}^n$, $n\ge2$;
unlabeled data $\mathcal U=\{\widetilde x_j\}_{j=1}^N$, $N\ge1$;
predictor $f$.
\Require Maximum depth $D$, minimum child size $m\ge2$, fixed quantile levels $u_1,\ldots,u_K$, and $\alpha\in(0,1)$.
\State Construct $\mathcal S^{(j)}$ for each coordinate from the full unlabeled sample using the rule above.
\State Compute
\[
    \hat\lambda\gets
    \frac{\widehat{\mathrm{Cov}}_n(Y,f(X))}
    {(1+n/N)\widehat{\mathrm{Var}}_{n+N}(f(X))},
\]
setting $\hat\lambda=0$ if the denominator is zero.
\State Set $r_i^{(\hat\lambda)}\gets y_i-\hat\lambda f(x_i)$ for
$i=1,\ldots,n$.
\Function{BuildMeanTree}{$\mathcal R,h$}
    \If{$h=D$ or no admissible candidate remains}
        \State \Return a leaf with region $\mathcal R$.
    \EndIf
    \State For each admissible $(j,s)$ and $a\in\{L,R\}$, compute $n_{(j,s,a)}$, $p_{(j,s,a)}$, and the sample residual variance $\widehat\sigma_{(j,s,a)}^2$ in $\mathcal R_a(j,s)$, as in \Cref{sec:semi-supervised_tree}.
    \State Choose $(j^*,s^*)$ minimizing
    \[
    \mathrm{VMS}(j,s)
    =\sum_{a\in\{L,R\}}
    p_{(j,s,a)}^2\frac{\widehat\sigma_{(j,s,a)}^2}{n_{(j,s,a)}}
    \]
    over admissible candidates, using the fixed tie-breaking rule.
    \State $T_L\gets$
    \Call{BuildMeanTree}{$\mathcal R_L(j^*,s^*),h+1$}.
    \State $T_R\gets$
    \Call{BuildMeanTree}{$\mathcal R_R(j^*,s^*),h+1$}.
    \State \Return an internal node with split $(j^*,s^*)$ and
    subtrees $T_L,T_R$.
\EndFunction
\State $T\gets$\Call{BuildMeanTree}{$\mathcal X,0$}.
\State For the leaves $\{\mathcal R_\ell\}_{\ell=1}^L$ of $T$, compute each leaf's labeled count $n_\ell$, unlabeled proportion $p_\ell$, residual mean $\overline r_\ell^{(\hat\lambda)}$, and sample residual variance $\widehat\sigma_{\ell,\hat\lambda}^2$, as in
\Cref{sec:semi-supervised_tree}.
\State Compute
\[
\begin{aligned}
\hat\mu_{T,\hat\lambda}
&\gets\frac{1}{N}\sum_{j=1}^N\hat\lambda f(\widetilde x_j)
+\sum_{\ell=1}^L p_\ell\overline r_\ell^{(\hat\lambda)},\\
\widehat\sigma_{T,\hat\lambda}^2
&\gets\sum_{\ell=1}^L
p_\ell^2\frac{\widehat\sigma_{\ell,\hat\lambda}^2}{n_\ell}.
\end{aligned}
\]
\State \Return $\hat\mu_{T,\hat\lambda}$ and the interval
\[
    \left[
    \hat\mu_{T,\hat\lambda}-z_{1-\alpha/2}\widehat\sigma_{T,\hat\lambda},
    \quad
    \hat\mu_{T,\hat\lambda}+z_{1-\alpha/2}\widehat\sigma_{T,\hat\lambda}
    \right].
\]
\end{algorithmic}
\end{algorithm}

\begin{algorithm}[H]
\caption{\partest Linear Regression Estimation and Inference}
\label{alg:semi-supervised_tree_linear_regression}
\small
\begin{algorithmic}[1]
\Require Labeled data $\mathcal L=\{(x_i,y_i)\}_{i=1}^n$, $n\ge2$;
unlabeled data $\mathcal U=\{\widetilde x_j\}_{j=1}^N$, $N\ge1$;
predictor $f$; target coordinate $k\in\{1,\ldots,d\}$.
\Require Maximum depth $D$, minimum child size $m\ge2$, fixed quantile
levels $u_1,\ldots,u_K$, and $\alpha\in(0,1)$.
\State Compute
$\widetilde\Sigma\gets N^{-1}\sum_{j=1}^N
\widetilde x_j\widetilde x_j^\top$.
\If{$\widetilde\Sigma$ is singular}
    \State \Return the estimate $0$ and the interval $\mathbb R$.
\EndIf
\State Compute
$\widetilde\theta\gets\widetilde\Sigma^{-1}
N^{-1}\sum_{j=1}^N\widetilde x_jf(\widetilde x_j)$.
\State Set $S_i\gets x_i(y_i-f(x_i))$ for $i=1,\ldots,n$.
\State Construct $\mathcal S^{(j)}$ for each coordinate from the full
unlabeled sample using the rule above.
\Function{BuildRegressionTree}{$\mathcal R,h$}
    \If{$h=D$ or no admissible candidate remains}
        \State \Return a leaf with region $\mathcal R$.
    \EndIf
    \State For each admissible $(j,s)$ and child $\mathcal R_a(j,s)$, compute $\widehat V_{\mathcal R_a(j,s)}^{(k,k)}$, the sample variance of $(\widetilde\Sigma^{-1}S_i)^{(k)}$ in that child.
    \State Choose $(j^*,s^*)$ minimizing
    \[
        \mathrm{VMS}_{\mathrm{LR}}(j,s,\mathcal R)
        =\sum_{a\in\{L,R\}}
        p_{\mathcal R_a(j,s)}^2
        \frac{\widehat V_{\mathcal R_a(j,s)}^{(k,k)}}
        {n_{\mathcal R_a(j,s)}}
    \]
    over admissible candidates, using the fixed tie-breaking rule.
    \State $T_L\gets$
    \Call{BuildRegressionTree}{$\mathcal R_L(j^*,s^*),h+1$}.
    \State $T_R\gets$
    \Call{BuildRegressionTree}{$\mathcal R_R(j^*,s^*),h+1$}.
    \State \Return an internal node with split $(j^*,s^*)$ and
    subtrees $T_L,T_R$.
\EndFunction
\State $T\gets$\Call{BuildRegressionTree}{$\mathcal X,0$}.
\State For the leaves $\{\mathcal R_\ell\}_{\ell=1}^L$ of $T$, compute
$n_{\mathcal R_\ell}$,
$p_{\mathcal R_\ell}$, the mean score $\overline S_{\mathcal R_\ell}$,
the correction $\widehat R_{\mathcal R_\ell}
=\widetilde\Sigma^{-1}\overline S_{\mathcal R_\ell}$, and
$\widehat V_{\mathcal R_\ell}^{(k,k)}$, as in
\Cref{sec:semi-supervised_tree_linear_regression}.
\State Compute
\[
\begin{aligned}
\widehat\theta_T^{(k)}
&\gets\widetilde\theta^{(k)}
+\sum_{\ell=1}^L p_{\mathcal R_\ell}
\widehat R_{\mathcal R_\ell}^{(k)},\\
\widehat\sigma^2
&\gets\sum_{\ell=1}^L p_{\mathcal R_\ell}^2
\frac{\widehat V_{\mathcal R_\ell}^{(k,k)}}{n_{\mathcal R_\ell}}.
\end{aligned}
\]
\State \Return $\widehat\theta_T^{(k)}$ and the interval
\[
    \left[
    \widehat\theta_T^{(k)}-z_{1-\alpha/2}\widehat\sigma,
    \quad
    \widehat\theta_T^{(k)}+z_{1-\alpha/2}\widehat\sigma
    \right].
\]
\end{algorithmic}
\end{algorithm}

\begin{algorithm}[H]
\caption{\parq Mean Estimation and Inference with Linear Interpolation}
\label{alg:parq}
\small
\begin{algorithmic}[1]
\Require Labeled data $\mathcal L=\{(x_i,y_i)\}_{i=1}^n$, $n\ge3$,
with distinct scalar features; unlabeled data
$\mathcal U=\{\widetilde x_j\}_{j=1}^N$, $N\ge2$; predictor $f$.
\Require Inflation factor $\kappa>1$ and $\alpha\in(0,1)$, both fixed.
\State Sort the labeled observations by feature value, obtaining
$(x_{(i)},y_{(i)})_{i=1}^n$ with $x_{(1)}<\cdots<x_{(n)}$.
\State Set $r_i\gets y_{(i)}-f(x_{(i)})$ for $i=1,\ldots,n$.
\State Set $s_i\gets(r_{i+1}-r_i)/(x_{(i+1)}-x_{(i)})$ for
$i=1,\ldots,n-1$.
\State Compute
\[
    \widehat L_2\gets2\kappa\max_{1\le i\le n-2}
    \left|\frac{s_{i+1}-s_i}{x_{(i+2)}-x_{(i)}}\right|.
\]
\For{$j=1,\ldots,N$}
    \If{$\widetilde x_j<x_{(1)}$}
        \State $i\gets1$.
    \ElsIf{$\widetilde x_j\ge x_{(n)}$}
        \State $i\gets n-1$.
    \Else
        \State Find $i$ such that
        $x_{(i)}\le\widetilde x_j<x_{(i+1)}$.
    \EndIf
    \State $q_1(\widetilde x_j)\gets
    r_i+s_i(\widetilde x_j-x_{(i)})$.
    \State $w_1(\widetilde x_j)\gets
    \tfrac12|(\widetilde x_j-x_{(i)})(\widetilde x_j-x_{(i+1)})|$.
    \State $Z_j\gets f(\widetilde x_j)+q_1(\widetilde x_j)$.
\EndFor
\State Compute
\[
\begin{aligned}
\widehat\mu_{\mathrm{PAQ}}
&\gets\frac{1}{N}\sum_{j=1}^N Z_j,\\
\widehat\sigma_{\mathrm{PAQ}}^2
&\gets\frac{1}{N(N-1)}\sum_{j=1}^N
(Z_j-\widehat\mu_{\mathrm{PAQ}})^2,\\
\widehat B_1
&\gets\frac{\widehat L_2}{N}\sum_{j=1}^Nw_1(\widetilde x_j).
\end{aligned}
\]
\State Set
\[
    C_{n,N}\gets
    \left[
    \widehat\mu_{\mathrm{PAQ}}
    -z_{1-\alpha/2}\widehat\sigma_{\mathrm{PAQ}}-\widehat B_1,
    \quad
    \widehat\mu_{\mathrm{PAQ}}
    +z_{1-\alpha/2}\widehat\sigma_{\mathrm{PAQ}}+\widehat B_1
    \right].
\]
\State \Return $\widehat\mu_{\mathrm{PAQ}}$ and $C_{n,N}$.
\end{algorithmic}
\end{algorithm}

The \partest mean and regression coverage results are given in \Cref{thm:tree-ci-coverage} and \Cref{thm:tree-lr-ci-coverage}, respectively. The \parq coverage result is given in \Cref{thm:nn_uni}, with extensions to nonuniform sampling and higher-degree interpolation in \Cref{sec:NN-1}.

\section{Motivating Example for PART}
\label{sec:motex}

We illustrate the variance reduction in \Cref{cor:part-variance-reduction} with a simple example. This example shows how partitioning can improve precision beyond the global correction used by \ppiplus. Afterwards, we also provide a simple example where partitioning provably does not help reduce variance.

\begin{example}[Strict asymptotic variance reduction]
\label{example:better estimator}
Let $X=(X^{(1)},X^{(2)})$, where the two coordinates are independent and each equals $-1$ or $1$ with probability $1/2$. Let $\varepsilon$ be independent of $X$, with $\E[\varepsilon]=0$ and $\Var(\varepsilon)=1$, and define
\[
    Y=X^{(1)}+X^{(2)}+\varepsilon,
    \qquad
    f(X)=X^{(2)}.
\]
Then $\mu=\E[Y]=0$. 
\paragraph{Global Residual Correction.}
Since $\Cov(Y,f(X))=\Var(f(X))=1$, the population \ppiplus coefficient is
\[
    \lambda^*=1, \qquad r^{(\lambda^*)}(X,Y)=X^{(1)}+\varepsilon.
\]
Thus the asymptotic variance of \ppiplus is
\[
    \Var\!\left(r^{(\lambda^*)}(X,Y)\right)=2.
\]

\paragraph{Corrections Within Leaves.}
Consider the partition according to the first coordinate, with leaves
\[
    \mathcal R_-=\{x:x^{(1)}\le-1\},
    \qquad
    \mathcal R_+=\{x:x^{(1)}>-1\}.
\]
Each leaf has probability $1/2$. For $s\in\{-1,1\}$,
\[
    \E\!\left[ r^{(\lambda^*)}(X,Y)\mid X^{(1)}=s\right]=s, 
    \qquad
    \Var\!\left(r^{(\lambda^*)}(X,Y)\mid X^{(1)}=s\right)=1.
\]
The probability-weighted within-leaf residual variance is therefore
\[
    v_0=\frac{1}{2}\cdot1+\frac{1}{2}\cdot1=1.
\]
The reduction of the variance from two to one is from the variance of the leaf mean residuals, which equal $-1$ and $1$ with equal probability.

\paragraph{Selection by PART.}
Take maximum depth $D=1$, quantile levels $u_1=1/4$ and $u_2=3/4$, and any fixed minimum child size $m\ge2$. The population candidate thresholds are $-1$ and $1$ in each coordinate. Only the threshold $-1$ is feasible, since the threshold $1$ produces an empty right child. Splitting on the first coordinate has population VMS score one, as calculated above. Splitting on the second coordinate has score two, since the residual $X^{(1)}+\varepsilon$ is independent of $X^{(2)}$. Thus the population greedy tree $T_0$ uniquely selects the first coordinate and has leaves $\mathcal R_-$ and $\mathcal R_+$.

The empirical quantiles equal their population counterparts with probability tending to one. On this event, the infeasible splits have empty right children. Each child of a feasible population split has probability $1/2$, so its labeled count eventually exceeds the fixed minimum $m$ with probability tending to one. The remaining moment and nondegeneracy conditions also hold. Hence \Cref{as:part-regularity} is satisfied, and \Cref{thm:tree-ci-coverage} gives
\[
    \sqrt{n}\bigl(\hat\mu_{T,\hat\lambda}-\mu\bigr)
    \xrightarrow{d}\mathcal N(0,1),
    \qquad
    \sqrt{n}\bigl(\hat\mu_{\mathrm{PPI++}}-\mu\bigr)
    \xrightarrow{d}\mathcal N(0,2).
\]
For a fixed confidence level, the variance consistency in \Cref{thm:tree-ci-coverage} implies that the ratio of the \partest interval width to the \ppiplus interval width converges in probability to $1/\sqrt2$, a reduction of approximately $29\%$.
\end{example}
In this example, the predictor accounts for $X^{(2)}$, while the global residual correction still includes variation from $X^{(1)}$.
PART estimates the mean residual separately in the two leaves and weights these corrections using the unlabeled sample. This removes the between-leaf contribution to the asymptotic variance, leaving the variation from $\varepsilon$. 

Finally, we conclude by providing an example where \partest does not improve over \ppiplus. Suppose that $Y=f(X)+\epsilon$, where $\epsilon\sim\mathcal N(0,\sigma^2)$ is independent of $X$, and $\operatorname{Var}(f(X))>0$. Then $\lambda^*=1$, so the residual is
\[
    r^{(\lambda^*)}(X,Y)=Y-f(X)=\epsilon.
\]
Every feature region of positive probability has residual mean zero and residual variance $\sigma^2$. Therefore, regardless of the feature partition, \partest will have a variance of $\sigma^2$, which is equal to the global residual variance. Thus, partitioning offers no variance reduction in this example.

\section{PART Estimator: Implementation and Analysis}\label{appx:proof of coverage}
\subsection{Estimating the Mean}\label{sec:semi-supervised_tree}

In this section, we describe the construction of \partest for estimating $\mu=\E[Y]$. Let $\mathcal{L}=\{(x_i,y_i)\}_{i=1}^n$ be the labeled sample, let $\mathcal{U}=\{\widetilde{x}_j\}_{j=1}^N$ be the unlabeled sample, and let $f:\mathcal{X}\to\mathbb{R}$ be the predictor. As in \Cref{sec:part}, compute
\[
    \hat{\lambda}
    \coloneqq
    \frac{\widehat{\mathrm{Cov}}_n\!\bigl(Y,f(X)\bigr)}
    {(1+n/N)\,\widehat{\mathrm{Var}}_{n+N}\!\bigl(f(X)\bigr)},
    \qquad
    r^{(\hat{\lambda})}(x_i,y_i)
    \coloneqq y_i-\hat{\lambda}f(x_i).
\]
The covariance uses the labeled sample, and the prediction variance uses the pooled labeled and unlabeled features. For any region $\mathcal{R}\subseteq\mathcal{X}$, define its empirical unlabeled mass by
\[
    \widehat p_N(\mathcal{R}) \coloneqq \frac{1}{N}\sum_{j=1}^N \mathbf{1}\{\widetilde{x}_j\in\mathcal{R}\}.
\]
We write $p(\mathcal{R})=P_X(\mathcal{R})$ to denote the population mass of the region $\mathcal{R}$.

\paragraph{Candidate Splitting Points.}
Fix a maximum depth $D$, a minimum number of labeled observations per child $m\ge2$, and quantile levels $0<u_1<\cdots<u_K<1$, with $K\ge2$. These choices are fixed independently of $n$ and $N$. For each coordinate $k\in\{1,\ldots,d\}$, compute the empirical quantiles from the full unlabeled sample:
\[
    q_j^{(k)}
    \coloneqq
    \inf\left\{t\in\mathbb R:
    \frac{1}{N}\sum_{i=1}^N
    \mathbf 1\{\widetilde x_i^{(k)}\le t\}\ge u_j
    \right\},
    \qquad j=1,\ldots,K.
\]
We use these quantile values directly as candidate thresholds:
\[
    \mathcal S^{(k)} \coloneqq \{q_j^{(k)}:j=1,\ldots,K\}.
\]
We remove duplicate thresholds within each coordinate. These sets are computed once and used at every node. Under condition \emph{(ii)} in \Cref{as:part-regularity}, corresponding empirical and population candidate regions agree up to vanishing population mass, and their feasibility decisions agree with probability tending to one.

\begin{remark}[Sufficient conditions for \Cref{as:part-regularity}]
We briefly justify the sufficient conditions stated after \Cref{as:part-regularity} that enable the greedy fitted tree to stabilize as the number of samples increases. Let $F_k$ and $F_{N,k}$ be the population and empirical marginal distribution functions of coordinate $k$, and let $q_{0,j}^{(k)}=F_k^{-1}(u_j)$. If $F_k$ is continuous, then $F_k(q_{0,j}^{(k)})=u_j$ and $F_{N,k}(q_j^{(k)})=\lceil Nu_j\rceil/N$ almost surely. Hence
\[
\begin{aligned}
    &P_X\!\left(
    \{x:x^{(k)}\le q_j^{(k)}\}
    \triangle
    \{x:x^{(k)}\le q_{0,j}^{(k)}\}
    \right) =
    \left|F_k(q_j^{(k)})-u_j\right|
    \le
    \sup_t|F_{N,k}(t)-F_k(t)|+\frac{1}{N}.
\end{aligned}
\]
By the DKW inequality, the expected supremum is $O(N^{-1/2})$ \citep{DKW}. Now, let $B_N$ be the set of feature values whose assignment differs under at least one pair of corresponding empirical and population
thresholds. Since the grid is fixed, $\E[p(B_N)]=O(N^{-1/2})$. Thus, the disagreement in the candidate regions shrinks, which proves the first part of condition (ii). Next, positive-probability population children eventually have at least $m$ labeled observations in their empirical counterparts by region consistency and the law of large numbers.

Finally, suppose $n/\sqrt N\to0$. By independence of the labeled and unlabeled samples and a union bound,
\[
    \Pr\!\left[\exists i\le n:x_i\in B_N\right] \le n\,\E[p(B_N)]=O(n/\sqrt N)\longrightarrow0.
\]
Thus, with probability tending to one, replacing population thresholds with empirical thresholds does not change which labeled observations belong to each candidate region. On this event, the empirical counterpart of a child with zero population probability also contains no labeled observations. Together with the argument above for children with positive probability, this establishes the feasibility part of condition (ii).

Alternatively, suppose the feature distribution has a density positive almost everywhere on a rectangular support. It is then straightforward to verify condition (ii): quantile consistency ensures that corresponding regions differ by vanishing probability mass, while the rectangular support preserves feasibility with probability tending to one. Thus, we can establish condition (ii) without requiring $n/\sqrt N\to0$.
\end{remark}

\paragraph{Splitting Function.}
At a node with region $\mathcal{R}$, a candidate $(k,s)$ with
$s\in\mathcal{S}^{(k)}$ produces the children
\[
    \mathcal{R}_L(k,s) \coloneqq\{x\in\mathcal{R}:x^{(k)}\le s\},
    \qquad
    \mathcal{R}_R(k,s) \coloneqq\{x\in\mathcal{R}:x^{(k)}>s\}.
\]
For $a\in\{L,R\}$, let $n_{(k,s,a)}$ be the labeled count in $\mathcal{R}_a(k,s)$, and let
\[
    p_{(k,s,a)}\coloneqq\widehat p_N\!\left(\mathcal{R}_a(k,s)\right).
\]
A candidate is admissible if both children contain at least $m$ labeled observations. For each admissible candidate, let $\overline r_{(k,s,a)}^{(\hat{\lambda})}$ be the mean residual in child $a$, and compute its sample residual variance
\[
    \widehat{\sigma}_{(k,s,a)}^2
    \coloneqq
    \frac{1}{n_{(k,s,a)}-1}
    \sum_{i:x_i\in\mathcal{R}_a(k,s)}
    \left(
    r^{(\hat{\lambda})}(x_i,y_i)
    -\overline r_{(k,s,a)}^{(\hat{\lambda})}
    \right)^2.
\]
Choose an admissible candidate that minimizes the \emph{Variance-Minimizing Split} (VMS) criterion
\[
    \mathrm{VMS}(k,s) \coloneqq \sum_{a\in\{L,R\}} p_{(k,s,a)}^2\, \frac{\widehat{\sigma}_{(k,s,a)}^2}{n_{(k,s,a)}},
\]
using a fixed deterministic rule to break ties.
\paragraph{Termination Criterion.}
Start at the root $\mathcal{X}$, at depth zero, and apply the splitting rule recursively to both children. Stop at depth $D$ or when no admissible candidate remains. Let $T$ be the resulting tree, with leaves
$\{\mathcal{R}_\ell\}_{\ell=1}^L$ and $L\le2^D$. For each leaf, let $n_\ell$ be its labeled count and define
\[
    p_\ell\coloneqq\widehat p_N(\mathcal{R}_\ell),
    \qquad
    \overline r_\ell^{(\hat{\lambda})}
    \coloneqq \frac{1}{n_\ell}\sum_{i:x_i\in\mathcal{R}_\ell} r^{(\hat{\lambda})}(x_i,y_i).
\]
The \partest estimator is given by
\[
    \hat{\mu}_{T,\hat{\lambda}} \coloneqq \frac{1}{N}\sum_{j=1}^N\hat{\lambda}f(\widetilde{x}_j)
    +
    \sum_{\ell=1}^L p_\ell\, \overline r_\ell^{(\hat{\lambda})}.
\]
We provide pseudocode of the estimator in \Cref{alg:semi-supervised_tree}.

\subsubsection{Constructing Confidence Intervals}
For each leaf, compute the sample residual variance
\[
    \widehat{\sigma}_{\ell,\hat{\lambda}}^2
    \coloneqq
    \frac{1}{n_\ell-1}
    \sum_{i:x_i\in\mathcal{R}_\ell}
    \left(
    r^{(\hat{\lambda})}(x_i,y_i)
    -\overline r_\ell^{(\hat{\lambda})}
    \right)^2.
\]
Then, we estimate the variance of $\hat{\mu}_{T,\hat{\lambda}}$ by
\[
    \widehat{\sigma}_{T,\hat{\lambda}}^2 \coloneqq \sum_{\ell=1}^L p_\ell^2\, \frac{\widehat{\sigma}_{\ell,\hat{\lambda}}^2}{n_\ell}.
\]
For a fixed $\alpha\in(0,1)$, the nominal $(1-\alpha)$ Wald interval is
\[
    \left[
    \hat{\mu}_{T,\hat{\lambda}}
    -z_{1-\alpha/2}\widehat{\sigma}_{T,\hat{\lambda}},
    \quad
    \hat{\mu}_{T,\hat{\lambda}}
    +z_{1-\alpha/2}\widehat{\sigma}_{T,\hat{\lambda}}
    \right],
\]
where $z_{1-\alpha/2}$ is the $(1-\alpha/2)$ standard normal quantile. Under \Cref{as:part-regularity}, \Cref{thm:tree-ci-coverage} establishes that the coverage probability of this interval converges to $1-\alpha$.

\subsection{Establishing Theoretical Coverage Properties}
\label{sec:part-coverage-proof}

We prove \Cref{thm:tree-ci-coverage} under \Cref{as:part-regularity}. Let $T_0$ be the population greedy tree defined in \Cref{sec:part}, with leaves $\{\mathcal{R}_{0,\ell}\}_{\ell=1}^{L_0}$. Recall that
\[
    \lambda^*\coloneqq\frac{\Cov(Y,f(X))}{\Var(f(X))}, \qquad r^{(\lambda^*)}(x,y)\coloneqq y-\lambda^*f(x),
\]
and
\[
    v_0\coloneqq \sum_{\ell=1}^{L_0}p(\mathcal{R}_{0,\ell}) \Var\!\left(r^{(\lambda^*)}(X,Y)\mid X\in\mathcal{R}_{0,\ell}\right)>0.
\]

\treecicoverage*

\begin{proof}
We break the proof into three steps. 

\noindent\textbf{1. Consistency of the selected tree.}
The law of large numbers and $n/N\to0$ give $\hat\lambda\xrightarrow{p}\lambda^*$. For each feasible population split, condition \emph{(ii)} and uniform laws of large numbers over rectangles ensure that the labeled and unlabeled child proportions converge to the corresponding population probabilities. The averages of $Y,f(X),Y^2,Yf(X),f(X)^2$ within each empirical child also converge to their corresponding population averages, since these quantities are integrable under condition \emph{(i)} and the population children have positive probability \citep[Chapter~19]{Vaart1998}. Together with $\hat\lambda\xrightarrow{p}\lambda^*$, this gives
convergence of the child residual means and variances to their population counterparts at $\lambda^*$. Consequently, the empirical score $n\,\mathrm{VMS}$ converges to the population splitting criterion for each such candidate.

At each splitting node, the candidate set is finite and the population minimizer is unique by condition \emph{(iii)}. The minimum is therefore separated from the other feasible values by a positive gap, so the chosen children differ from the population children by vanishing population mass. Condition \emph{(ii)} also ensures agreement of the stopping decisions. Repeating this argument through the fixed depth gives $\Pr(L=L_0)\to1$ and, after indexing corresponding leaves in the same order, gives
\[
    P_X(\mathcal{R}_\ell\triangle\mathcal{R}_{0,\ell}) \xrightarrow{p}0, \qquad \ell=1,\ldots,L_0,
\]
where $\triangle$ denotes symmetric difference. Define the population mean residual in leaf $\ell$ by
\[
    \mu_{\ell,\lambda^*} \coloneqq \E\!\left[r^{(\lambda^*)}(X,Y)\mid X\in\mathcal{R}_{0,\ell}\right].
\]
All population leaves have positive mass. The same moment convergence therefore gives
\begin{equation}
\begin{aligned}
    p_\ell,\;n_\ell/n
    &\xrightarrow{p}p(\mathcal{R}_{0,\ell})>0,\\
    \overline r_\ell^{(\hat\lambda)}
    &\xrightarrow{p}\mu_{\ell,\lambda^*},\\
    \widehat\sigma_{\ell,\hat\lambda}^2
    &\xrightarrow{p}
    \Var\!\left(r^{(\lambda^*)}(X,Y)\mid X\in\mathcal{R}_{0,\ell}\right).
\end{aligned}
\label{eq:part-leaf-consistency}
\end{equation}

\medskip
\noindent\textbf{2. Asymptotic normality.}
The key step is an exact rewriting of the fitted estimator. Since
\[
    \sum_{\ell=1}^{L}\frac{n_\ell}{n} \overline r_\ell^{(\hat\lambda)} =\frac1n\sum_{i=1}^n\bigl(y_i-\hat\lambda f(x_i)\bigr),
\]
we have
\begin{equation}
\begin{aligned}
    \hat\mu_{T,\hat\lambda}
    &=\frac{1}{n}\sum_{i=1}^n y_i
    +\hat\lambda\left[
    \frac{1}{N}\sum_{j=1}^N f(\widetilde x_j)
    -\frac{1}{n}\sum_{i=1}^n f(x_i)\right]+
    \sum_{\ell=1}^{L}
    \left(p_\ell-\frac{n_\ell}{n}\right)
    \overline r_\ell^{(\hat\lambda)}.
\end{aligned}
\label{eq:part-mean-decomposition}
\end{equation}
Thus PART equals the \ppiplus estimator plus a correction for differences between the labeled and unlabeled leaf proportions. To control these differences, define the centered labeled count
\[
    Z_n(\mathcal R) \coloneqq \frac{1}{\sqrt{n}}\sum_{i=1}^n \left(\mathbf1\{x_i\in\mathcal R\}-p(\mathcal R)\right).
\]
Rectangle indicators in fixed dimension form a bounded VC class and hence are $P_X$-Donsker \citep[Chapter~19]{Vaart1998}. In particular,
\[
    \sup_{\mathcal R}|Z_n(\mathcal R)|=O_p(1),
\]
so the scaled sampling errors are controlled uniformly over all rectangles, including the fitted leaves.  Applying \citet[Lemma~19.24]{Vaart1998} to the corresponding indicator functions and using the region convergence from Step~1 yields
\[
    Z_n(\mathcal R_\ell)-Z_n(\mathcal R_{0,\ell})=o_p(1).
\]
Thus the fitted-leaf centered counts can be replaced by their population-leaf counterparts, even when the fitted leaves are selected using the labeled responses.

Since $n/N\to0$ and $f(X)$ has finite variance, we have
\begin{subequations}
\label{eq:part-unlabeled-bounds}
\begin{equation}
\label{eq:part-unlabeled-prediction-bound}
\sqrt{n}\left[ \frac{1}{N}\sum_{j=1}^N f(\widetilde x_j)-\E[f(X)] \right] =\sqrt{\frac{n}{N}}\,O_p(1) =o_p(1).
\end{equation}

For the leaf proportions, the analogous uniform count bound for the unlabeled sample gives $\sqrt N[p_\ell-p(\mathcal R_\ell)]=O_p(1)$, where $p_\ell=\widehat p_N(\mathcal R_\ell)$. Adding and subtracting the population leaf probability $p(\mathcal R_\ell)$, we obtain
\begin{equation}
\label{eq:part-leaf-proportion-bound}
\begin{aligned}
    \sqrt{n}\left(p_\ell-\frac{n_\ell}{n}\right)
    &=
    \sqrt{n}\bigl[p_\ell-p(\mathcal R_\ell)\bigr]
    -\sqrt{n}\left[\frac{n_\ell}{n}-p(\mathcal R_\ell)\right]\\
    &=
    -Z_n(\mathcal R_\ell)
    +\sqrt{\frac {n}{N}}\,
     \sqrt{N}\bigl[p_\ell-p(\mathcal R_\ell)\bigr]\\
    &=
    -Z_n(\mathcal R_\ell)+o_p(1).
\end{aligned}
\end{equation}
\end{subequations}

By Step 1, we may work on the event $L=L_0$, with corresponding leaves indexed as above, since this event has probability tending to one. Now, centering \eqref{eq:part-mean-decomposition} gives
\begin{align*}
\sqrt{n}\bigl(\hat\mu_{T,\hat\lambda}-\mu\bigr)
&=\frac{1}{\sqrt{n}}\sum_{i=1}^n(y_i-\mu)
  +\hat\lambda\sqrt{n}
   \left[
     \frac{1}{N}\sum_{j=1}^N f(\tilde x_j)-\E[f(X)]
   \right]\\
&\quad
  -\hat\lambda\frac{1}{\sqrt{n}}\sum_{i=1}^n
     \bigl(f(x_i)-\E[f(X)]\bigr)
  +\sum_{\ell=1}^{L_0}
     \overline r_\ell^{(\hat\lambda)}
     \sqrt n\left(p_\ell-\frac{n_\ell}{n}\right).
\end{align*}
Since $\hat\lambda=O_p(1)$, $\overline r_\ell^{(\hat\lambda)}=O_p(1)$, and $L_0$ is fixed, substituting \eqref{eq:part-unlabeled-bounds} into this expansion yields
\begin{align*}
\sqrt{n}\bigl(\hat\mu_{T,\hat\lambda}-\mu\bigr)
&=\frac{1}{\sqrt{n}}\sum_{i=1}^n
   \left[
     y_i-\mu-\hat\lambda\bigl(f(x_i)-\E[f(X)]\bigr)
   \right]
  -\sum_{\ell=1}^{L_0}
     \overline r_\ell^{(\hat\lambda)}Z_n(\mathcal R_\ell)
  +o_p(1).
\end{align*}

The centered labeled prediction sum and the $Z_n$ terms are $O_p(1)$. Thus, by the consistency results in Step 1, replacing $\hat\lambda$ by $\lambda^*$ and $\overline r_\ell^{(\hat\lambda)}$ by $\mu_{\ell,\lambda^*}$ contributes only $o_p(1)$. Using also $Z_n(\mathcal R_\ell)-Z_n(\mathcal R_{0,\ell})=o_p(1)$ and $\E[r^{(\lambda^*)}(X,Y)]=\mu-\lambda^*\E[f(X)]$, we obtain
\begin{align*}
\sqrt{n}\bigl(\hat\mu_{T,\hat\lambda}-\mu\bigr)
&=\frac{1}{\sqrt{n}}\sum_{i=1}^n
   \left[
     y_i-\mu-\lambda^*\bigl(f(x_i)-\E[f(X)]\bigr)
   \right]
  -\sum_{\ell=1}^{L_0}
     \mu_{\ell,\lambda^*}Z_n(\mathcal R_\ell)
  +o_p(1)\\[0.5em]
&=\frac{1}{\sqrt{n}}\sum_{i=1}^n
   \left[
     r^{(\lambda^*)}(x_i,y_i)
     -\E[r^{(\lambda^*)}(X,Y)]
   \right]
  -\sum_{\ell=1}^{L_0}
     \mu_{\ell,\lambda^*}Z_n(\mathcal R_{0,\ell})
  +o_p(1).
\end{align*}

Finally, expand the definition of $Z_n$. Since the population leaves form a partition,
$\sum_{\ell=1}^{L_0}
p(\mathcal R_{0,\ell})\mu_{\ell,\lambda^*}
=\E[r^{(\lambda^*)}(X,Y)]$,
so the centering terms cancel. In particular,
\begin{align*}
\sqrt{n}\bigl(\hat\mu_{T,\hat\lambda}-\mu\bigr)
&=\frac{1}{\sqrt{n}}\sum_{i=1}^n
   \left[
     r^{(\lambda^*)}(x_i,y_i)
     -\E[r^{(\lambda^*)}(X,Y)]
   \right]\\
&\quad
  -\frac{1}{\sqrt{n}}\sum_{i=1}^n\sum_{\ell=1}^{L_0}
     \mu_{\ell,\lambda^*}
     \left[
       \mathbf1\{x_i\in\mathcal R_{0,\ell}\}
       -p(\mathcal R_{0,\ell})
     \right]
  +o_p(1)\\[0.5em]
&=\frac{1}{\sqrt{n}}\sum_{i=1}^n
   \left[
     r^{(\lambda^*)}(x_i,y_i)
     -\sum_{\ell=1}^{L_0}
        \mu_{\ell,\lambda^*}
        \mathbf1\{x_i\in\mathcal R_{0,\ell}\}
   \right]
  +o_p(1).
\end{align*}

The summands in the final expression are i.i.d., have mean zero, and have variance $v_0$. The central limit theorem therefore proves
\[
    \sqrt{n}(\hat\mu_{T,\hat\lambda}-\mu) \xrightarrow{d}\mathcal N(0,v_0).
\]

\medskip
\noindent\textbf{3. Variance consistency and coverage.}
Using \eqref{eq:part-leaf-consistency},
\[
\begin{aligned}
    n\widehat\sigma_{T,\hat\lambda}^2
    &=\sum_{\ell=1}^{L}
    \frac{p_\ell^2}{n_\ell/n}\widehat\sigma_{\ell,\hat\lambda}^2 \xrightarrow{p}
    \sum_{\ell=1}^{L_0}p(\mathcal R_{0,\ell})
    \Var\!\left(r^{(\lambda^*)}(X,Y)\mid X\in\mathcal R_{0,\ell}\right)
    =v_0.
\end{aligned}
\]
Since $v_0>0$, Slutsky's theorem gives
\[
    \frac{\hat\mu_{T,\hat\lambda}-\mu} {\widehat\sigma_{T,\hat\lambda}}\xrightarrow{d}\mathcal N(0,1),
\]
which proves the stated coverage limit for every fixed
$\alpha\in(0,1)$.
\end{proof}

\subsection{Estimating Linear Regression Coefficients}
\label{sec:semi-supervised_tree_linear_regression}

We now adapt \partest to estimate a linear regression coefficient. The estimator combines an initial coefficient vector computed from model predictions on the unlabeled sample with a correction estimated from labeled prediction errors. As in mean estimation, we compute the correction within tree leaves and combine the leaf corrections using their unlabeled sample proportions. The partition is chosen using a variance criterion for the coefficient of interest.
Let
\[
    \Sigma\coloneqq\E[XX^\top], \qquad \theta\coloneqq\Sigma^{-1}\E[XY],
\]
where $\Sigma$ is positive definite. The vector $\theta$ is the population least-squares coefficient. We fix a target coordinate $k\in\{1,\ldots,d\}$ of $\theta$, which we would like to estimate.

Let $\widetilde X$ be the unlabeled design matrix with rows $\widetilde x_j^\top$. The empirical second-moment matrix and prediction-based coefficient are defined by
\[
    \widetilde\Sigma \coloneqq\frac{1}{N}\sum_{j=1}^N\widetilde x_j\widetilde x_j^\top,
    \qquad
    \widetilde\theta \coloneqq\widetilde\Sigma^{-1} \frac{1}{N}\sum_{j=1}^N\widetilde x_j f(\widetilde x_j),
\]
respectively. Here, $\widetilde\theta$ is the ordinary least-squares coefficient obtained by regressing the predicted responses $f(\widetilde x_j)$ on the unlabeled covariates $\widetilde x_j$. The correction needed to recover the target coefficient follows from the decomposition
\[
    \theta = \Sigma^{-1}\E[Xf(X)] + \Sigma^{-1}\E[X\{Y-f(X)\}].
\]
We estimate the first term by $\widetilde\theta$. The second term accounts for prediction errors and is estimated using the labeled sample.
\paragraph{Leaf Corrections and Variances.}
To estimate the residual term in the above decomposition, define the score correction vector for each labeled observation by
\[
    S_i\coloneqq x_i\bigl(y_i-f(x_i)\bigr).
\]
For a region $\mathcal R$, let $n_{\mathcal R}$ be its labeled count, and write
\[
    p_{\mathcal R}\coloneqq\widehat p_N(\mathcal R)=\frac{1}{N}\sum_{j=1}^N\mathbf1\{\widetilde x_j\in\mathcal R\},
    \qquad
    p(\mathcal R)\coloneqq P_X(\mathcal R).
\]
Define the mean score and correction vector in the region by
\[
    \overline S_{\mathcal R} \coloneqq\frac{1}{n_{\mathcal R}} \sum_{i:x_i\in\mathcal R}S_i,
    \qquad
    \widehat R_{\mathcal R} \coloneqq\widetilde\Sigma^{-1}\overline S_{\mathcal R}.
\]
For $n_{\mathcal R}\ge2$, the sample covariance of the score vectors and its transformed version are
\[
    \widehat M_{\mathcal R}
    \coloneqq\frac{1}{n_{\mathcal R}-1}
    \sum_{i:x_i\in\mathcal R}
    (S_i-\overline S_{\mathcal R})
    (S_i-\overline S_{\mathcal R})^\top,
    \qquad
    \widehat V_{\mathcal R}
    \coloneqq
    \widetilde\Sigma^{-1}\widehat M_{\mathcal R}
    \widetilde\Sigma^{-1}.
\]
Thus $\widehat V_{\mathcal R}^{(k,k)}$ is the sample variance of the $k$th coordinates of the transformed scores $\widetilde\Sigma^{-1}S_i$ in the region.

\paragraph{Splitting Function for Linear Regression.}
We use the empirical quantile values in \Cref{sec:semi-supervised_tree} as candidate thresholds, together with the same minimum child size $m\ge2$, deterministic tie-breaking rule, and termination criterion. For a candidate split $(j,s)$ in region $\mathcal R$, let $\mathcal R_L(j,s)$ and $\mathcal R_R(j,s)$ be its children. Among candidates with at least $m$ labeled observations in each child, choose one minimizing the \emph{Variance-Minimizing Split} criterion
\[
    \mathrm{VMS}_{\mathrm{LR}}(j,s,\mathcal R) \coloneqq \sum_{a\in\{L,R\}} p_{\mathcal R_a(j,s)}^2 \frac{\widehat V_{\mathcal R_a(j,s)}^{(k,k)}}{n_{\mathcal R_a(j,s)}}.
\]
Let $T$ be the resulting tree, with leaves $\{\mathcal R_\ell\}_{\ell=1}^L$ and $L\leq2^D$. The \partest regression estimator is
\[
    \widehat\theta_T \coloneqq\widetilde\theta +\sum_{\ell=1}^L p_{\mathcal R_\ell}\widehat R_{\mathcal R_\ell}.
\]
We use its $k$th coordinate $\widehat\theta_T^{(k)}$ to estimate $\theta^{(k)}$.

\subsubsection{Constructing Confidence Intervals}
Define
\[
    \widehat\sigma^2 \coloneqq \sum_{\ell=1}^L p_{\mathcal R_\ell}^2 \frac{\widehat V_{\mathcal R_\ell}^{(k,k)}}{n_{\mathcal R_\ell}}.
\]
For a fixed $\alpha\in(0,1)$, the nominal $(1-\alpha)$ Wald interval is
\[
    \left[
    \widehat\theta_T^{(k)}-z_{1-\alpha/2}\widehat\sigma,
    \quad
    \widehat\theta_T^{(k)}+z_{1-\alpha/2}\widehat\sigma
    \right],
\]
where $z_{1-\alpha/2}$ is the $(1-\alpha/2)$ standard normal quantile. We establish asymptotic coverage in
\Cref{thm:asymptotic distribution for regression}.

\subsection{Theoretical Coverage Guarantees}
\label{sec:part-lr-coverage}

For a region $\mathcal R$ with positive population probability, define the population correction and covariance by
\[
\begin{aligned}
    R_{\mathcal R}
    &\coloneqq
    \Sigma^{-1}\E\!\left[X\bigl(Y-f(X)\bigr)\mid X\in\mathcal R\right],\\
    V_{\mathcal R}
    &\coloneqq
    \Sigma^{-1}
    \Cov\!\left(X\bigl(Y-f(X)\bigr)\mid X\in\mathcal R\right)
    \Sigma^{-1}.
\end{aligned}
\]
For the fixed coordinate $k$, let $T_0$ be the population regression tree obtained by using the corresponding population quantiles as candidate thresholds, admitting splits with two positive-probability children, and minimizing
\[
\sum_{a\in\{L,R\}}
p(\mathcal R_a(j,s))\,V_{\mathcal R_a(j,s)}^{(k,k)}
\]
at each node. The population procedure uses the same quantile levels and depth limit as its empirical counterpart and stops when no admissible split remains.

\begin{assumption}[Regularity for PART linear regression]
\label{as:part-lr-regularity}
The dimension $d$, maximum depth $D$, minimum child size $m\geq2$, finite grid of candidate quantile levels, and target coordinate $k$ are fixed. Assume:
\begin{enumerate}[label=(\roman*),leftmargin=*]

\item The labeled observations are i.i.d. from $P_{X,Y}$, the unlabeled observations are i.i.d. from $P_X$, and the two samples are independent. As $n,N\to\infty$, $n/N\to0$. Furthermore, $\Sigma$ is positive definite and
\[
    \E\!\left[\|X\|^4+(1+\|X\|^2)\bigl(Y^2+f(X)^2\bigr)\right]<\infty.
\]

\item The candidate-region and feasibility conditions in part (ii) of \Cref{as:part-regularity} hold for the regression population tree $T_0$ defined above.

\item At every splitting node of $T_0$, the population regression splitting criterion has a unique minimizer among the feasible candidates. Moreover,
\[
    v_{0,k}\coloneqq\sum_{\mathcal R\text{ a leaf of }T_0}p(\mathcal R)V_{\mathcal R}^{(k,k)}>0.
\]
\end{enumerate}
\end{assumption}

The moment condition in \emph{(i)} is a sufficient condition for finite variances of the regression score vectors. For bounded covariates, it reduces to $\E[Y^2+f(X)^2]<\infty$. Condition \emph{(ii)} requires corresponding empirical and population candidate regions to disagree on a set of vanishing population probability, and their feasibility decisions to agree with probability tending to one. The sufficient conditions discussed after \Cref{as:part-regularity} apply here as well. Together with consistency of the split scores, the unique-minimizer requirement in \emph{(iii)} makes the empirical split choices stabilize. The quantity $v_{0,k}$ is the probability-weighted within-leaf variance of the $k$th coordinate of $\Sigma^{-1}X(Y-f(X))$; its positivity ensures a nondegenerate limiting variance. As in mean estimation, $n/N\to0$ makes unlabeled sampling uncertainty negligible at the $\sqrt{n}$ scale under the moment condition.

\begin{restatable}[Asymptotic coverage of PART for linear regression]
{theorem}{treelr}
\label{thm:asymptotic distribution for regression}
\label{thm:tree-lr-ci-coverage}
Under \Cref{as:part-lr-regularity}, the data-dependent tree $T$ and
estimates constructed above satisfy
\[
\sqrt{n}\bigl(\widehat\theta_T^{(k)}-\theta^{(k)}\bigr)\xrightarrow{d}\mathcal N(0,v_{0,k}),\qquad n\widehat\sigma^2\xrightarrow{p}v_{0,k}.
\]
Consequently, for every fixed $\alpha\in(0,1)$,
\[
    \Pr\!\left[
    \theta^{(k)}\in
    \left[
    \widehat\theta_T^{(k)}-z_{1-\alpha/2}\widehat\sigma,
    \quad
    \widehat\theta_T^{(k)}+z_{1-\alpha/2}\widehat\sigma
    \right]
    \right]\longrightarrow1-\alpha.
\]
\end{restatable}

\begin{proof}
We break the proof into three steps.

\noindent\textbf{1. Consistency of the selected tree.}
The law of large numbers gives $\widetilde\Sigma\xrightarrow{p}\Sigma$. Since $\Sigma$ is positive definite, $\widetilde\Sigma$ is invertible with probability tending to one and
\[
    \widetilde\Sigma^{-1}\xrightarrow{p}\Sigma^{-1}.
\]
The moment condition makes the coordinates of $X(Y-f(X))$ square-integrable. Uniform laws of large numbers over rectangles
therefore apply to their first and second moments \citep[Chapter~19]{Vaart1998}. Together with condition \emph{(ii)}, these laws imply that, for each feasible population split, the empirical child proportions, score means, and score covariance matrices converge to their population counterparts. Consequently, $n\,\mathrm{VMS}_{\mathrm{LR}}$ converges to the population regression splitting criterion.

The candidate set is finite, so each unique population minimizer is separated from the other feasible values by a positive gap. The same finite-depth argument as in the proof of
\Cref{thm:tree-ci-coverage} then gives $\Pr(L=L_0)\to1$ and
\[
    P_X(\mathcal R_\ell\triangle\mathcal R_{0,\ell}) \xrightarrow{p}0,\qquad \ell=1,\ldots,L_0,
\]
where $\{\mathcal R_{0,\ell}\}_{\ell=1}^{L_0}$ are the leaves of $T_0$, indexed to correspond to the fitted leaves. Condition \emph{(ii)} also ensures that the stopping decisions agree. All population leaves have positive mass, and the same moment convergence gives
\begin{equation}
\begin{aligned}
p_{\mathcal R_\ell},\quad n_{\mathcal R_\ell}/n
&\xrightarrow{p}p(\mathcal R_{0,\ell})>0,\\
\widehat R_{\mathcal R_\ell}
&\xrightarrow{p}R_{\mathcal R_{0,\ell}},\\
\widehat V_{\mathcal R_\ell}
&\xrightarrow{p}V_{\mathcal R_{0,\ell}}.
\end{aligned}
\label{eq:part-lr-leaf-consistency}
\end{equation}

\medskip
\noindent\textbf{2. Asymptotic normality.}
As in mean estimation, we rewrite the estimator using the difference between labeled and unlabeled leaf proportions. To begin, observe that
\[
    \sum_{\ell=1}^L\frac{n_{\mathcal R_\ell}}{n}\widehat R_{\mathcal R_\ell}=\widetilde\Sigma^{-1}\frac{1}{n}\sum_{i=1}^n S_i. 
\]
Next, we add and subtract the correction weighted by the labeled leaf proportions. Using the preceding identity gives
\begin{align*}
    \widehat\theta_T-\theta &=
    \widetilde\theta-\theta +\sum_{\ell=1}^L \frac{n_{\mathcal R_\ell}}{n} \widehat R_{\mathcal R_\ell} + \sum_{\ell=1}^L
    \left(p_{\mathcal R_\ell}-\frac{n_{\mathcal R_\ell}}{n} \right)
    \widehat R_{\mathcal R_\ell}\\
    &= \widetilde\theta-\theta +\widetilde\Sigma^{-1}\frac{1}{n}\sum_{i=1}^n S_i + \sum_{\ell=1}^L
    \left(p_{\mathcal R_\ell}-\frac{n_{\mathcal R_\ell}}{n} \right) \widehat R_{\mathcal R_\ell}.
\end{align*}

We next express $\widetilde\theta-\theta$ in terms of the unlabeled sample. Since $\theta=\widetilde\Sigma^{-1}\widetilde\Sigma\theta$, the definition of $\widetilde\theta$ yields
\begin{align*}
    \widetilde\theta-\theta &= \widetilde\Sigma^{-1}
    \left[ \frac{1}{N}\sum_{j=1}^N \widetilde x_j f(\widetilde x_j) -\widetilde\Sigma\theta \right] = \widetilde\Sigma^{-1}
    \frac{1}{N}\sum_{j=1}^N \widetilde x_j \bigl(f(\widetilde x_j)-\widetilde x_j^\top\theta \bigr),
\end{align*}
where the second equality uses 
$\widetilde\Sigma =N^{-1}\sum_{j=1}^N \widetilde x_j\widetilde x_j^\top$.

By Step 1, we may work on the event $L=L_0$, with corresponding leaves indexed as above, since this event has probability tending
to one. Substituting this expression and collecting the terms multiplied by $\widetilde\Sigma^{-1}$ gives
\begin{equation}
\label{eq:part-lr-decomposition}
\begin{aligned}
    \widehat\theta_T-\theta
    &=
    \widetilde\Sigma^{-1} \left[ \frac{1}{n}\sum_{i=1}^n S_i +\frac{1}{N}\sum_{j=1}^N \widetilde x_j \bigl( f(\widetilde x_j)-\widetilde x_j^\top\theta \bigr)
    \right]
    +
    \sum_{\ell=1}^L
    \left( p_{\mathcal R_\ell}-\frac{n_{\mathcal R_\ell}}{n} \right)
    \widehat R_{\mathcal R_\ell}.
\end{aligned}
\end{equation}
The final sum accounts for the difference between the unlabeled and labeled leaf proportions. We next analyze the bracketed term in
\eqref{eq:part-lr-decomposition}.
The definition of $\theta$ gives
\[
\E[X(Y-f(X))] +\E[X(f(X)-X^\top\theta)]=0.
\]
Consequently,
\[
\E[X(f(X)-X^\top\theta)] = -\E[X(Y-f(X))] = -\E[S_i].
\]
Under condition \emph{(i)}, the summands in the second empirical average in \eqref{eq:part-lr-decomposition} have a finite
covariance matrix. Since $n/N\to0$, the $\sqrt{n}$-scaled sampling error of this average is $\sqrt{n/N}\,O_p(1)=o_p(1)$. Replacing it by its expectation therefore yields
\begin{align*}
    &\sqrt{n}\left[
    \frac{1}{n}\sum_{i=1}^n S_i
    +\frac{1}{N}\sum_{j=1}^N
    \widetilde x_j
    \bigl(f(\widetilde x_j)-\widetilde x_j^\top\theta\bigr)
    \right]\\
    &\qquad=
    \frac{1}{\sqrt{n}}\sum_{i=1}^n
    \bigl(S_i-\E[S_i]\bigr)
    +o_p(1).
\end{align*}
Since $\widetilde\Sigma^{-1}=O_p(1)$ by Step 1, the resulting error in the $\sqrt n$-scaled decomposition is also $o_p(1)$. Next, define the centered labeled count
\[
    Z_n(\mathcal R) \coloneqq \frac{1}{\sqrt{n}}\sum_{i=1}^n \left(\mathbf1\{x_i\in\mathcal R\}-p(\mathcal R)\right).
\]
The rectangle-count bounds used in the mean proof give
\[
    \sup_{\mathcal R}|Z_n(\mathcal R)|=O_p(1),
    \qquad
    Z_n(\mathcal R_\ell)-Z_n(\mathcal R_{0,\ell})=o_p(1),
\]
where the supremum is over axis-aligned rectangles. The second relation follows from the region convergence in Step~1
and \citet[Lemma~19.24]{Vaart1998}. The corresponding uniform bound for the unlabeled counts and $n/N\to0$ give
\[
    \sqrt{n}\left( p_{\mathcal R_\ell}-\frac{n_{\mathcal R_\ell}}n \right) =-Z_n(\mathcal R_\ell)+o_p(1)
\]
as in the \partest mean proof.  We now expand \eqref{eq:part-lr-decomposition} to yield
\begin{align*}
\sqrt{n}(\widehat\theta_T-\theta)
&=
\widetilde\Sigma^{-1}
\frac{1}{\sqrt{n}}\sum_{i=1}^n
\bigl(S_i-\E[S_i]\bigr) +
\sum_{\ell=1}^{L_0}
\sqrt{n}\left(
p_{\mathcal R_\ell}-\frac{n_{\mathcal R_\ell}}{n}
\right)
\widehat R_{\mathcal R_\ell}
+o_p(1)\\[0.5em]
&=
\widetilde\Sigma^{-1}
\frac{1}{\sqrt{n}}\sum_{i=1}^n
\bigl(S_i-\E[S_i]\bigr)-
\sum_{\ell=1}^{L_0}
\widehat R_{\mathcal R_\ell}Z_n(\mathcal R_\ell)
+o_p(1),
\end{align*}
where the second equality uses that  $\sqrt{n}(p_{\mathcal R_\ell}-n_{\mathcal R_\ell}/n) =-Z_n(\mathcal R_\ell)+o_p(1)$, together with $\widehat R_{\mathcal R_\ell}=O_p(1)$ and the fact that $L_0$ is fixed.

By the moment condition, it holds that  $n^{-1/2}\sum_{i=1}^n(S_i-\E[S_i])=O_p(1)$ and the uniform count bound gives $\sup_{\mathcal R}|Z_n(\mathcal R)|=O_p(1)$. Thus, consistency of $\widetilde\Sigma^{-1}$ and the fitted leaf corrections allows us to replace them by their population counterparts. We then use $Z_n(\mathcal R_\ell)-Z_n(\mathcal R_{0,\ell})=o_p(1)$ to replace the fitted-leaf counts:
\begin{align*}
    \sqrt{n}(\widehat\theta_T-\theta)
    &=
    \Sigma^{-1}
    \frac{1}{\sqrt{n}}\sum_{i=1}^n
    \bigl(S_i-\E[S_i]\bigr)-
    \sum_{\ell=1}^{L_0}
    R_{\mathcal R_{0,\ell}}Z_n(\mathcal R_\ell)
    +o_p(1)\\[0.5em]
    &=
    \Sigma^{-1}
    \frac{1}{\sqrt{n}}\sum_{i=1}^n
    \bigl(S_i-\E[S_i]\bigr)-
    \sum_{\ell=1}^{L_0}
    R_{\mathcal R_{0,\ell}}Z_n(\mathcal R_{0,\ell})
    +o_p(1).
\end{align*}
Since the population leaves form a partition, the definition of $R_{\mathcal R}$ and the law of total expectation give
\[
\sum_{\ell=1}^{L_0} p(\mathcal R_{0,\ell})R_{\mathcal R_{0,\ell}} = \Sigma^{-1}\E[S_i].
\]
Expanding the definition of $Z_n$ and using this identity to cancel the centering terms yields
\begin{align*}
    \sqrt{n}(\widehat\theta_T-\theta)
    &=
    \frac{1}{\sqrt{n}}\sum_{i=1}^n
    \Sigma^{-1}\bigl(S_i-\E[S_i]\bigr)\\
    &\quad-
    \frac{1}{\sqrt{n}}\sum_{i=1}^n
    \sum_{\ell=1}^{L_0}
    R_{\mathcal R_{0,\ell}}
    \left[
    \mathbf1\{x_i\in\mathcal R_{0,\ell}\}
    -p(\mathcal R_{0,\ell})
    \right]
    +o_p(1)\\[0.5em]
    &=
    \frac{1}{\sqrt{n}}\sum_{i=1}^n
    \left[
    \Sigma^{-1}S_i
    -\sum_{\ell=1}^{L_0}
    R_{\mathcal R_{0,\ell}}
    \mathbf1\{x_i\in\mathcal R_{0,\ell}\}
    \right]
    +o_p(1).
\end{align*}

Taking the $k$th coordinate therefore gives
\[
    \begin{aligned}
    \sqrt{n}\bigl(
    \widehat\theta_T^{(k)}-\theta^{(k)}
    \bigr)
    &=
    \frac{1}{\sqrt{n}}\sum_{i=1}^n
    \left[
    \bigl(\Sigma^{-1}S_i\bigr)^{(k)}
    -\sum_{\ell=1}^{L_0}
    R_{\mathcal R_{0,\ell}}^{(k)}
    \mathbf1\{x_i\in\mathcal R_{0,\ell}\}
    \right]
    +o_p(1).
    \end{aligned}
\]
These summands are i.i.d., have mean zero, and have variance $v_{0,k}$. The central limit theorem proves asymptotic normality.

\medskip
\noindent\textbf{3. Variance consistency and coverage.}
By \eqref{eq:part-lr-leaf-consistency},
\[
    n\widehat\sigma^2
    =\sum_{\ell=1}^L
    \frac{p_{\mathcal R_\ell}^2}{n_{\mathcal R_\ell}/n}
    \widehat V_{\mathcal R_\ell}^{(k,k)}
    \xrightarrow{p}
    \sum_{\ell=1}^{L_0}
    p(\mathcal R_{0,\ell})V_{\mathcal R_{0,\ell}}^{(k,k)}
    =v_{0,k}.
\]
Since $v_{0,k}>0$, Slutsky's theorem gives
\[
    \frac{\widehat\theta_T^{(k)}-\theta^{(k)}}{\widehat\sigma} \xrightarrow{d}\mathcal N(0,1),
\]
which proves the coverage statement.
\end{proof}

\subsection{Estimating Linear Relationship between Income and Sex on Census Data}\label{sec:lrcensus}
In this section, we empirically test the performance of \partest in a  linear regression task. We follow the same testing methodology defined in \Cref{sec:experiments}. Our goal is to use census data from \citet{Ding2021} to estimate a linear relationship between sex (M/F) and individual income. For predictions, we utilize an XGBoost model trained by \citep{ppi++}. We plot the confidence interval width and coverage probabilities in \Cref{fig:lrcensus}. In \Cref{tab:master_results}, we show that with $n=600$ labeled samples, \partest reduces the interval width by 4.8\% while maintaining valid coverage.

\begin{figure}[H]
        \centering
        \begin{minipage}{0.5\textwidth}
        \centering
        \includegraphics[width=\linewidth]{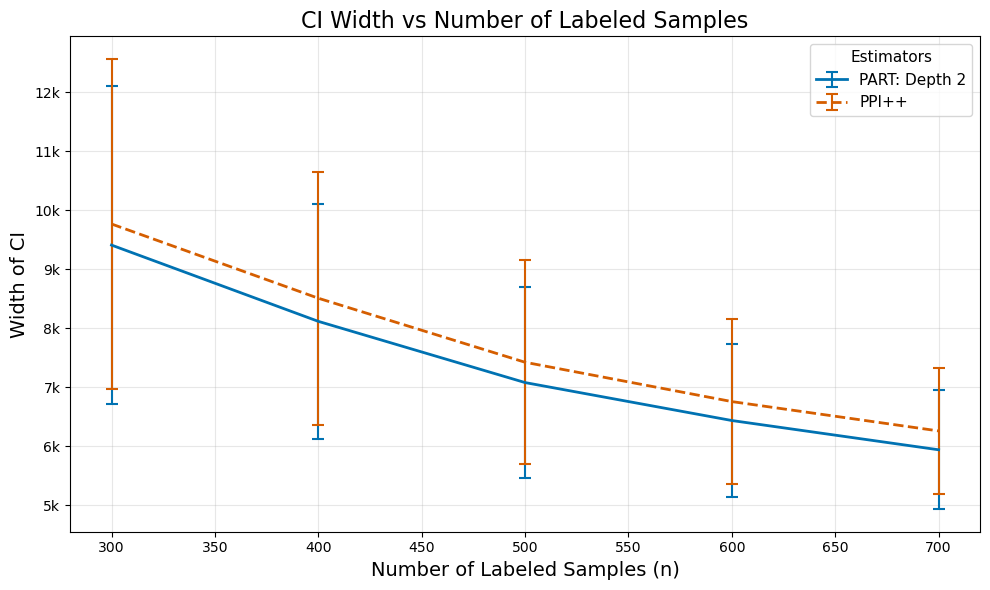}
        \end{minipage}\hfill
        \begin{minipage}{0.5\textwidth}
            \centering
            \includegraphics[width=\linewidth]{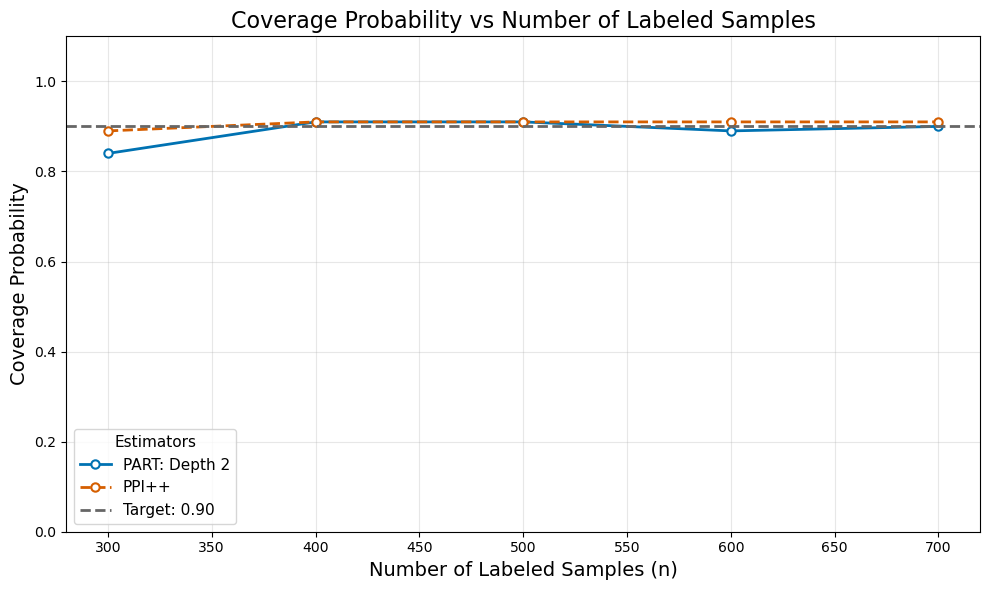}
        \end{minipage}
        \caption{Estimating the linear regression coefficient between income and sex (M/F) on California census data \cite{Ding2021}.}
        \label{fig:lrcensus}
\end{figure}

\section{PAQ Estimator: Implementation and Analysis}\label{sec:NN-1}

\subsection{Estimating the Mean}
\label{sec:nn:nonuniform}\label{nn:degp}
In this section, we describe the implementation of \parq, an estimator based on polynomial interpolation that achieves better convergence rates than $\ppiplus$ and $\partest$. We derive confidence intervals and establish coverage guarantees using a general polynomial interpolation framework. The result for linear interpolation under uniform features, described in \Cref{thm:nn_uni}, follows as a special case. We begin with some assumptions on the residual function that help control the interpolation error.

\begin{assumption}[Regularity for Polynomial PAQ]\label{as:smooth_p}
Fix an integer $p\ge1$. Assume:
\begin{enumerate}[label=(\roman*),leftmargin=*]
\item The labeled observations are i.i.d. from $P_{X,Y}$, the unlabeled observations are i.i.d. from $P_X$, and the two samples are independent.
\item The distribution $P_X$ has a density $\rho$ supported on a compact interval $[a,b]$, where $a<b$, with $\rho(x)\ge c>0$ for almost every $x\in[a,b]$.
\item There is a deterministic function $y$ such that $Y=y(X)$ almost surely. The residual function $r(x) = y(x)-f(x)$ belongs to $C^{p+1}([a,b])$, and $0<\Var(Y)<\infty$.
\end{enumerate}
\end{assumption}

Condition \emph{(i)} is standard for mean estimation. Condition \emph{(ii)} ensures that the labeled features become dense throughout the interval as the sample size grows. The estimator does not require knowledge of the feature density or its lower bound. Condition \emph{(iii)} assumes deterministic responses and requires $p+1$ continuous derivatives of the residual, which allows us to bound the error of degree-$p$ interpolation. As expected, higher degree interpolation requires stronger smoothness assumptions.
\paragraph{Local Polynomial Interpolation.} We begin by defining our interpolation procedure and the \parq estimator.
Let $n\ge p+2$ and $N\geq 2$. We sort the labeled features as $x_{(1)}<\cdots<x_{(n)}$ and write $r_i=r(x_{(i)})$. For each \(i=1,\ldots,n-1\), set
\[
    j_i := \min\{i,n-p\}.
\]
On the interval \([x_{(i)},x_{(i+1)}]\), define \(q_p\) to equal the unique polynomial of degree at most \(p\) that interpolates the residuals at the \(p+1\) consecutive labeled features
\[
    x_{(j_i)},\ldots,x_{(j_i+p)}.
\]
For \(x<x_{(1)}\), use the polynomial through the first \(p+1\) labeled features; for \(x>x_{(n)}\), use the polynomial through the last \(p+1\) labeled features. Thus, \(q_p\) is a piecewise polynomial interpolant with polynomial extrapolation at the two boundaries. For \(p=1\), this is exactly the linear construction in \Cref{sec:paq_method}. Define
\[
Z_j^p\coloneqq f(\widetilde x_j)+q_p(\widetilde x_j), \qquad \widehat\mu_{\mathrm{PAQ}}^p \coloneqq\frac{1}{N}\sum_{j=1}^N Z_j^p,
\]
and
\[
    (\widehat\sigma_{\mathrm{PAQ}}^p)^2 \coloneqq\frac{1}{N(N-1)} \sum_{j=1}^N(Z_j^p-\widehat\mu_{\mathrm{PAQ}}^p)^2.
\]
The rectified point $Z_j^p$ is formed by taking the initial prediction $f(\widetilde x_j)$ and correcting it with an interpolated value given by $q_p(\widetilde x_j)$.
\paragraph{Interpolation Error and Confidence Intervals.} We next construct a bound for the interpolation error we accrue from using $q_p(\widetilde{x}_j)$ in $\parq$. This bound will be used to help form confidence intervals for \parq. Let $t_0(x),\ldots,t_p(x)$ be the nodes used to compute $q_p(x)$. The Lagrange remainder formula gives
\[
    r(x)-q_p(x) = \frac{r^{(p+1)}(\xi_x)}{(p+1)!} \prod_{j=0}^p\bigl(x-t_j(x)\bigr),
\]
where $\xi_x$ lies in the smallest interval containing $x$ and its interpolation nodes. Thus, writing
\[
    L_{p+1}\coloneqq\|r^{(p+1)}\|_\infty, \qquad w_p(x)\coloneqq \frac{1}{(p+1)!}\prod_{j=0}^p|x-t_j(x)|,
\]
we obtain the pointwise error bound
\[
|r(x)-q_p(x)|\le L_{p+1}w_p(x).
\]
The factor $w_p(x)$ is directly computable from the selected nodes. The derivative bound $L_{p+1}$ is unknown and hence we must estimate it. To estimate $L_{p+1}$, we use divided differences, which generalize the slope between two nodes to higher orders and accommodate unequal node spacing. In particular, we may compute the divided difference recursively from the observed residuals:
\[
r[t_0]=r(t_0), \qquad r[t_0,\ldots,t_k] = \frac{r[t_1,\ldots,t_k]-r[t_0,\ldots,t_{k-1}]} {t_k-t_0}.
\]
By the mean-value theorem, it holds that
\[
    (p+1)!\,r[t_0,\ldots,t_{p+1}] = r^{(p+1)}(\eta)
\]
for some $\eta$ between the nodes. This motivates taking the largest absolute divided difference over blocks of $p+2$ consecutive labeled features and inflating it by a factor $\kappa$. Specifically,
\[
    \widehat L_{p+1}
    \coloneqq
    \kappa(p+1)!
    \max_{1\le i\le n-p-1}
    \left|r[x_{(i)},\ldots,x_{(i+p+1)}]\right|,
    \qquad \kappa>1,
\]
where $\kappa$ provides a margin that ensures $\Pr(\widehat L_{p+1}\ge L_{p+1})\to1$. For $p=1$, the previous expression reduces to the expression for $\widehat L_2$ in the main body. We now average the resulting pointwise error allowances over the unlabeled sample, yielding
\[
    \widehat B_p \coloneqq \frac{\widehat L_{p+1}}{N} \sum_{j=1}^N w_p(\widetilde x_j).
\]
On the event $\widehat L_{p+1}\ge L_{p+1}$, this quantity bounds the average interpolation error,
\[
    \left| \frac{1}{N}\sum_{j=1}^N \bigl(q_p(\widetilde x_j)-r(\widetilde x_j)\bigr) \right|
    \le \widehat B_p.
\]
Consequently, we can add this bound to the standard $z$-interval to produce the interval
\[
    C_{n,N}^p
    \coloneqq
    \left[
    \widehat\mu_{\mathrm{PAQ}}^p
    -z_{1-\alpha/2}\widehat\sigma_{\mathrm{PAQ}}^p-\widehat B_p,
    \quad
    \widehat\mu_{\mathrm{PAQ}}^p
    +z_{1-\alpha/2}\widehat\sigma_{\mathrm{PAQ}}^p+\widehat B_p
    \right].
\]
Below, we show that this interval achieves the desired coverage asymptotically. Furthermore, we also demonstrate improved convergence rates over both $\partest$ and $\ppi$.
\begin{theorem}[Polynomial PAQ under nonuniform sampling]
\label{thm:nn_gen_smooth}
Under \Cref{as:smooth_p}, as $n,N\to\infty$, with fixed $\kappa>1$ and
$\alpha\in(0,1)$,
\[
    \left|\E[\widehat\mu_{\mathrm{PAQ}}^p]-\mu\right|
    =O(n^{-(p+1)}),
    \qquad
    \E\!\left[(\widehat\mu_{\mathrm{PAQ}}^p-\mu)^2\right]
    =O(N^{-1}+n^{-2p-2}).
\]
The same bound holds for the variance. Moreover,
\[
    \liminf_{n,N\to\infty}\Pr(\mu\in C_{n,N}^p)\ge1-\alpha,
    \qquad
    |C_{n,N}^p|=O_p(N^{-1/2}+n^{-(p+1)}).
\]
\end{theorem}
\begin{proof}
Let $e_p(x)\coloneqq q_p(x)-r(x)$. Since $f+r=y$, the estimation error decomposes as
\[
    \widehat\mu_{\mathrm{PAQ}}^p-\mu = \left[ \frac{1}{N}\sum_{j=1}^N y(\widetilde x_j)-\mu \right]
    +
    \frac{1}{N}\sum_{j=1}^N e_p(\widetilde x_j).
\]
We first control the interpolation term, then use this decomposition to establish the estimation rates and confidence interval coverage. We break the proof into 4 steps.

\medskip
\noindent\textbf{1. Bounding expected interpolation error.}

Recall that the remainder bound gives
\[
    |e_p(x)|\le L_{p+1}w_p(x), \qquad x\in[a,b].
\]
Set $x_{(0)}=a$ and $x_{(n+1)}=b$, and consider the intervals
\[
    I_i=[x_{(i)},x_{(i+1)}], \qquad i=0,\ldots,n.
\]
For each $I_i$, let $J_i$ be the smallest interval containing $I_i$ and the interpolation nodes used there. The consecutive-node construction ensures that $J_i$ spans at most $p+1$ consecutive gaps, including on the two boundary intervals where extrapolation is used. Write
\[
    H_i\coloneqq |J_i|, \qquad V_i\coloneqq \int_{J_i}\rho(x)\,dx.
\]
Because $\rho\ge c>0$, we have $H_i\le V_i/c$. Moreover, for $x\in I_i$, every factor $|x-t_j(x)|$ in $w_p(x)$
is at most $H_i$. To control the mean-squared interpolation error, we bound the integrated square of $w_p$. Indeed, the remainder inequality gives $e_p(x)^2\le L_{p+1}^2w_p(x)^2$, so a bound on $\int_a^b w_p(x)^2\rho(x)\,dx$ also controls the average squared error of the interpolant. In particular, by the previous bounds, it holds that
\begin{align*}
    \int_{I_i}w_p(x)^2\rho(x)\,dx
    &=
    \int_{I_i}
    \left[
    \frac{1}{(p+1)!}
    \prod_{j=0}^{p}|x-t_j(x)|
    \right]^2
    \rho(x)\,dx\\
    &\le
    \int_{I_i}
    \left[
    \frac{H_i^{p+1}}{(p+1)!}
    \right]^2
    \rho(x)\,dx\\
    &=
    \frac{H_i^{2p+2}}{((p+1)!)^2}
    \int_{I_i}\rho(x)\,dx\\
    &\le
    \frac{(V_i/c)^{2p+2}}{((p+1)!)^2}\,V_i\\
    &=
    \frac{V_i^{2p+3}}
    {c^{2p+2}((p+1)!)^2}.
\end{align*}

To bound the expectation of this quantity, we apply the probability integral transform. Let $F$ be the distribution function of $X$, and set
\[
    U_j\coloneqq F(x_j), \qquad j=1,\ldots,n.
\]
Since $F$ is continuous, the transformed features $U_1,\ldots,U_n$ are independent $\operatorname{Unif}[0,1]$ random variables. Write $U_{(j)}=F(x_{(j)})$ for their order statistics, and set
\[
    U_{(0)}=0, \qquad U_{(n+1)}=1.
\]
For our consecutive-node construction, write
\[
    J_i=[x_{(s_i)},x_{(s_i+k_i)}], \qquad 1\le k_i\le p+1,
\]
where the indices are fixed by the interpolation rule. The probability mass of $J_i$ then becomes
\begin{align*}
    V_i
    &=\int_{x_{(s_i)}}^{x_{(s_i+k_i)}}\rho(x)\,dx\\
    &=F(x_{(s_i+k_i)})-F(x_{(s_i)})\\
    &=U_{(s_i+k_i)}-U_{(s_i)}\\
    &=\sum_{j=s_i}^{s_i+k_i-1}
    \bigl(U_{(j+1)}-U_{(j)}\bigr).
\end{align*}
Thus, $V_i$ is the sum of a fixed block of $k_i$ consecutive spacings from a uniform sample. The $n+1$ uniform spacings, including the two boundary spacings, jointly have a Dirichlet distribution with all parameters equal to one. By the aggregation property of the Dirichlet distribution, grouping the $k_i$ spacings in this block and the remaining $n+1-k_i$ spacings gives
\[
V_i\sim\operatorname{Beta}(k_i,n+1-k_i).
\]
Since $p$ is fixed, the beta moment formula gives
\[
\E[V_i^{2p+3}] = \prod_{m=0}^{2p+2}\frac{k_i+m}{n+1+m} = O(n^{-2p-3}),
\]
uniformly in $i$. Summing over the $n+1$ intervals and using the pointwise remainder bound yields
\begin{align*}
    \E\!\left[
    \int_a^b w_p(x)^2\rho(x)\,dx
    \right]
    =O(n^{-2p-2}),
    \end{align*}
    and hence it follows that
    \[
    \E\!\left[
    \int_a^b e_p(x)^2\rho(x)\,dx
    \right]
    =O(n^{-2p-2}).
\]
\medskip
\noindent\textbf{2. Bounding bias, mean-squared error, and variance.}
We now apply the previous bounds to control the error and variance of \parq. Conditional on the labeled sample, $e_p$ is fixed and the unlabeled features remain independent draws from $\rho$. Jensen's inequality and the preceding bound therefore give
\begin{align*}
    \E\!\left[
    \left(
    \frac{1}{N}\sum_{j=1}^N e_p(\widetilde x_j)
    \right)^2
    \right]
    \le
    \E\!\left[
    \frac{1}{N}\sum_{j=1}^N e_p(\widetilde x_j)^2
    \right]
    =
    \E\!\left[
    \int_a^b e_p(x)^2\rho(x)\,dx
    \right]
    =O(n^{-2p-2}).
\end{align*}

The first term in the estimation-error decomposition has mean zero and variance $\operatorname{Var}(Y)/N$. Hence, by Cauchy--Schwarz,
\begin{align*}
    \left|\E[\widehat\mu_{\mathrm{PAQ}}^p]-\mu\right|
    &\le
    \left[
    \E\!\left[
    \left(
    \frac{1}{N}\sum_{j=1}^N e_p(\widetilde x_j)
    \right)^2
    \right]
    \right]^{1/2}\\
    &=O(n^{-(p+1)}).
\end{align*}
Applying $(u+v)^2\le2u^2+2v^2$ to the same decomposition gives
\begin{align*}
\E\!\left[
(\widehat\mu_{\mathrm{PAQ}}^p-\mu)^2
\right]
&\le
\frac{2\operatorname{Var}(Y)}{N}
+
2\E\!\left[
\left(
\frac{1}{N}\sum_{j=1}^N e_p(\widetilde x_j)
\right)^2
\right]\\
&=O(N^{-1}+n^{-2p-2}).
\end{align*}
The variance satisfies the same bound because it is no larger than the mean-squared error.

\medskip
\noindent\textbf{3. Estimating the derivative bound.}
We next show that $\widehat L_{p+1}$ is at least $L_{p+1}$ with probability tending to one, which justifies its use in the interpolation allowance $\widehat B_p$. Write
\[
M_{p+1}\coloneqq\frac{\widehat L_{p+1}}{\kappa}.
\]
The mean-value theorem for divided differences gives
\[
0\le M_{p+1}\le L_{p+1}.
\]
For any $\varepsilon>0$, continuity of $r^{(p+1)}$ implies that there is an interval of positive length on which
\[
|r^{(p+1)}(x)|>L_{p+1}-\varepsilon.
\]
This interval has positive probability mass because $\rho\ge c>0$. It therefore contains a block of $p+2$ consecutive labeled features with probability tending to one. Applying the mean-value theorem to that block gives
\[
M_{p+1}\ge L_{p+1}-\varepsilon
\]
on this event. Thus,
\[
M_{p+1}\xrightarrow{p}L_{p+1}.
\]

If $L_{p+1}>0$, the fixed inflation factor $\kappa>1$ consequently ensures
\[
\Pr(\widehat L_{p+1}\ge L_{p+1})\longrightarrow1.
\]
If $L_{p+1}=0$, the same conclusion holds because $\widehat L_{p+1}=0$ for every sample. Since $M_{p+1}\le L_{p+1}$ and $\widehat L_{p+1}=\kappa M_{p+1}$, we also have
\[
0\le\widehat L_{p+1}\le\kappa L_{p+1}.
\]
We will use this upper bound to control the size of $\widehat B_p$ when establishing the rate for the confidence interval length.

\medskip
\noindent\textbf{4. Confidence interval coverage.}

We first verify that the estimated standard error has the correct asymptotic scale. The bound in Step~2 and Markov's inequality imply
\[
\frac{1}{N}\sum_{j=1}^N e_p(\widetilde x_j)^2=o_p(1).
\]
Since
\[
Z_j^p=f(\widetilde x_j)+q_p(\widetilde x_j)
=y(\widetilde x_j)+e_p(\widetilde x_j),
\]
Cauchy--Schwarz and the law of large numbers give
\begin{align*}
\left|
\frac{1}{N}\sum_{j=1}^N (Z_j^p)^2
-
\frac{1}{N}\sum_{j=1}^N y(\widetilde x_j)^2
\right|
&\leq
2\left(
\frac{1}{N}\sum_{j=1}^N y(\widetilde x_j)^2
\right)^{1/2}
\left(
\frac{1}{N}\sum_{j=1}^N e_p(\widetilde x_j)^2
\right)^{1/2}
+
\frac{1}{N}\sum_{j=1}^N e_p(\widetilde x_j)^2\\
&=o_p(1).
\end{align*}
Together with the consistency of $\widehat\mu_{\mathrm{PAQ}}^p$ established in Step~2, this yields
\begin{align*}
N(\widehat\sigma_{\mathrm{PAQ}}^p)^2
=
\frac{N}{N-1}
\left[
\frac{1}{N}\sum_{j=1}^N (Z_j^p)^2
-
(\widehat\mu_{\mathrm{PAQ}}^p)^2
\right] \xrightarrow{p}\operatorname{Var}(Y)>0.
\end{align*}

The central limit theorem applied to the average of $y(\widetilde x_j)$, followed by Slutsky's theorem, now gives
\[
\frac{
\frac{1}{N}\sum_{j=1}^N y(\widetilde x_j)-\mu
}{
\widehat\sigma_{\mathrm{PAQ}}^p
}
\ \Rightarrow\ \mathcal N(0,1).
\]
On the event $\widehat L_{p+1}\ge L_{p+1}$, the interpolation allowance bounds the difference between this average and the PAQ estimator as
\[
\left|
\widehat\mu_{\mathrm{PAQ}}^p
-\frac{1}{N}\sum_{j=1}^N y(\widetilde x_j)
\right|
\le\widehat B_p.
\]
The triangle inequality therefore implies
\begin{align*}
\Pr(\mu\in C_{n,N}^p)
\geq
\Pr\!\left(
\left|
\frac{1}{N}\sum_{j=1}^N y(\widetilde x_j)-\mu
\right|
\leq z_{1-\alpha/2}\widehat\sigma_{\mathrm{PAQ}}^p
\right) -
\Pr(\widehat L_{p+1}<L_{p+1}).
\end{align*}
The first probability converges to $1-\alpha$, and the second converges to zero, proving the coverage claim. Finally, the deterministic bound $\widehat L_{p+1}\le\kappa L_{p+1}$, Jensen's inequality, and Step~1 give
\begin{align*}
\E[\widehat B_p^2]
&\leq
\kappa^2L_{p+1}^2
\E\!\left[
\frac{1}{N}\sum_{j=1}^N w_p(\widetilde x_j)^2
\right]\\
&=
\kappa^2L_{p+1}^2
\E\!\left[
\int_a^b w_p(x)^2\rho(x)\,dx
\right]\\
&=O(n^{-2p-2}).
\end{align*}
Thus $\widehat B_p=O_p(n^{-(p+1)})$. Combining this with $\widehat\sigma_{\mathrm{PAQ}}^p=O_p(N^{-1/2})$, we obtain
\begin{align*}
|C_{n,N}^p|
=
2z_{1-\alpha/2}\widehat\sigma_{\mathrm{PAQ}}^p
+2\widehat B_p
=O_p(N^{-1/2}+n^{-(p+1)}).
\end{align*}
\end{proof}
It is straightforward to apply the previous result to establish \Cref{thm:nn_uni} from the main body.
\parqmain*

\begin{proof}
First, it is clear that \Cref{as:nn} implies \Cref{as:smooth_p}. Hence, we may apply \Cref{thm:nn_gen_smooth} with $p=1$, $[a,b]=[0,1]$, and $\rho=1$ to complete the proof.
\end{proof}

\subsection{Synthetic Experiment for the PAQ Estimator}\label{sec:NNExp}
In this section, we empirically test the \parq estimator against \ppiplus on a synthetic dataset that meets the assumptions of \Cref{thm:nn_gen_smooth}---specifically, that residuals are deterministic and twice continuously differentiable. In each of 100 trials, we independently generate $M=10{,}000$ observations with $X\sim\operatorname{Unif}[0,\pi]$, $Y=X^2+\sin(X)$, and $f(X)=X^2$. For each labeled sample size $n\in\{100,175,250,325,400\}$, we randomly select $n$ labeled observations and use the remaining $N=M-n$ observations as unlabeled data. We construct nominal 90\% intervals, use $\kappa=2$ for \parq, and evaluate coverage against $\mu=\pi^2/3+2/\pi$. Width error bars show one standard deviation. As illustrated in \Cref{fig:paq}, \parq yields tighter intervals than \ppiplus while maintaining nominal coverage.
    \begin{figure}[H]
        \centering
        \begin{minipage}{0.5\textwidth}
        \centering
        \includegraphics[width=\linewidth]{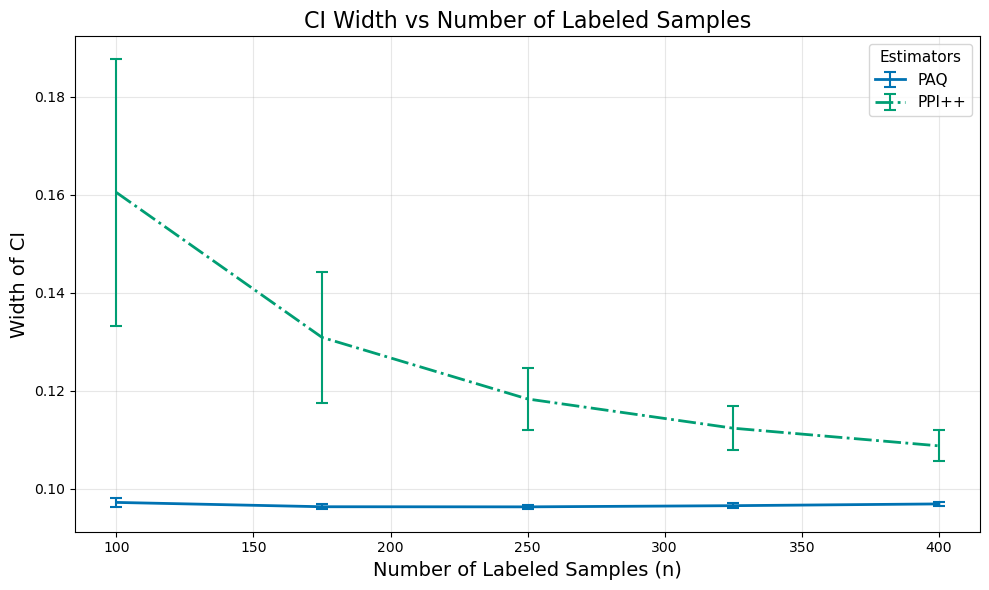}
        \end{minipage}\hfill
        \begin{minipage}{0.5\textwidth}
            \centering
            \includegraphics[width=\linewidth]{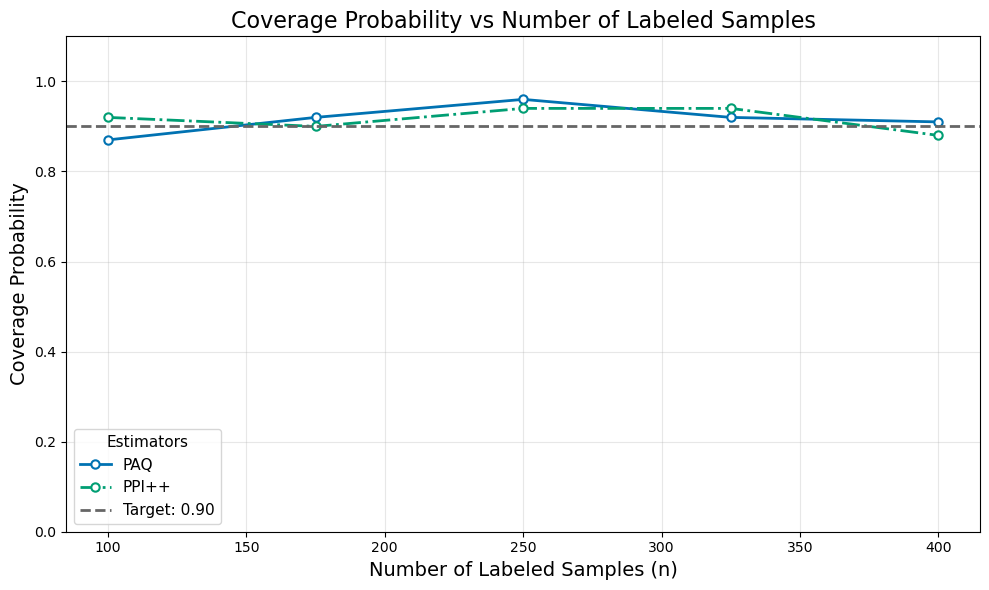}
        \end{minipage}
        \caption{Estimating the mean on a synthetic dataset.}
        \label{fig:paq}
    \end{figure}

\end{document}